\documentclass[10pt]{article}
\usepackage{thuidm}
\usepackage{algorithm}
\usepackage{tabularx}
\newcolumntype{Y}{>{\centering\arraybackslash}X}
\usepackage{amsmath}
\usepackage{bm}
\usepackage{algorithmic}
\usepackage{multirow}
\usepackage{adjustbox}
\usepackage{float}
\usepackage{subcaption}
\usepackage{wrapfig}

\AtEndPreamble{
  \theoremstyle{thuidm}
  \newtheorem{principle}[theorem]{Principle}
}

\newcommand{\K}{\mathcal{K}}
\newcommand{\D}{\mathcal{D}}

\newcommand{\E}{\mathbb{E}}
\newcommand{\T}{\mathcal{T}}

\newcommand{\V}{\mathcal{V}}
\newcommand{\loss}{\mathcal{L}}

\title{What Should LLMs Forget: \\Reinforcement Unlearning of Knowledge Associations}

\author{
  Wutong Xu\aff{1\dag}, Yun Qu\aff{1\dag}, Haiyan Feng\aff{1}, Qi Wang\aff{1}, Xiangyang Ji\aff{1}
}

\affiliation{
  \aff{1}Department of Automation, Tsinghua University, Beijing, China
}

\datestamp{October 30, 2026}
\correspondence{Qi Wang \href{mailto:cheemswang@mail.tsinghua.edu.cn}{cheemswang@mail.tsinghua.edu.cn} and Xiangyang Ji \href{mailto:xyji@tsinghua.edu.cn}{xyji@tsinghua.edu.cn}}

\labwebsite[www.thuidm.com]{https://www.thuidm.com/}

\begin{document}

\begin{abstract}
Large language models (LLMs) acquire extensive factual knowledge during pre-training, which may include sensitive, private, or copyrighted information, motivating the development of unlearning methods that selectively remove specific knowledge while preserving general capabilities.
However, existing approaches typically optimize entire responses without explicitly identifying the knowledge to be forgotten, resulting in coarse and uncontrolled forgetting.
To address this, we revisit the forgetting target of LLM unlearning and introduce Knowledge Association (KA), which characterizes the association between a query context and its target knowledge and is instantiated by the corresponding conditional predictive distribution, as the fundamental target of forgetting.
Based on KA, we propose Knowledge Association Reinforcement Unlearning (KARU), a unified reinforcement learning framework that suppresses forgotten knowledge associations while preserving retained ones.
Specifically, KARU derives complementary mechanisms from the KA-based formulation:
token-level reward construction for finer-grained forgetting supervision, distribution-level regularization for precise knowledge suppression, and proximity-aware retain sampling to mitigate collateral degradation.
Experiments across diverse LLMs on representative benchmarks demonstrate that KARU consistently improves forgetting effectiveness over existing baselines while maintaining comparable retention performance and general utility.
\end{abstract}

\maketitle

\begingroup
\renewcommand{\thefootnote}{}
\footnotetext{\textsuperscript{\dag}\,These authors contributed equally to this work.}
\endgroup

\section{Introduction}
\label{sec:intro}

Large language models (LLMs) acquire extensive factual knowledge during pre-training on vast web-scale corpora~\citep{brown2020language,carlini2022quantifying}. While this knowledge underpins their impressive capabilities across diverse tasks, it also introduces significant security and privacy risks: LLMs may memorize and reproduce sensitive personal information~\citep{carlini2021extractingtrainingdatalarge,parii2025machine}, copyrighted material~\citep{panaitescu2025can,liu2024survey}, or sensitive information related to weapons and cyberattacks~\citep{li2024wmdpbenchmarkmeasuringreducing,bengio2024managing}. These risks are particularly relevant to the Web ecosystem, where large-scale data collection and reuse can expose private or sensitive information through model outputs. They are further underscored by privacy regulations such as the General Data Protection Regulation (GDPR)~\citep{MANTELERO2013229}, which establishes individuals' ``right to be forgotten,'' as well as the growing need for secure, privacy-preserving, and trustworthy AI systems. In response to these challenges, \emph{LLM unlearning}~\citep{jang2022knowledgeunlearningmitigatingprivacy,eldan2023s,yao2024largelanguagemodelunlearning} has emerged as a critical research direction, aiming to mitigate these risks by removing the influence of specific knowledge or training data from LLMs while preserving their overall capabilities.

Various approaches, including supervised optimization, preference learning, and reinforcement learning, have been explored for LLM unlearning~\citep{jang2022knowledgeunlearningmitigatingprivacy,zhang2024npo,fan2025simplicityprevailsrethinkingnegative}. In particular, reinforcement learning (RL) formulates unlearning as a sequential decision-making problem and supports flexible reward design through online policy optimization~\citep{zhang2025rulereinforcementunlearningachieves}. Recent studies have demonstrated the potential of RL-based unlearning~\citep{jin2024rwkubenchmarkingrealworldknowledge,shi2024muse}.

Despite these advances, existing RL-based unlearning methods treat all generated content equally, without explicitly characterizing the target knowledge to be removed. Consequently, optimization signals are often distributed over irrelevant tokens, resulting in imprecise forgetting. Moreover, forgotten knowledge may remain encoded in the predictive distribution, leaving it vulnerable to extraction or reactivation~\citep{patil2023sensitiveinformationdeletedllms,wei2023jailbroken,yuan2025towards}. Meanwhile, retention objectives typically assign uniform importance to retain samples, overlooking their heterogeneous susceptibility to collateral degradation~\citep{dolgova2026forgetting,ha2025unlearning,cheng2026machine}. This observation raises a fundamental question: \emph{what exactly should an LLM forget during unlearning?}

In response to this question, we revisit the nature of unlearning targets and argue that existing methods implicitly optimize the underlying Knowledge Association (KA) rather than the entire generated response.
KA characterizes the association between a query context and its target knowledge and is instantiated by the conditional predictive distribution.
For instance, when unlearning where an author was born, the true target is the model's tendency to reveal the answer given the corresponding query, not the entire response of the forget sample.
Treating KA as the forgetting target, we formulate a unified optimization objective that minimizes the discrepancy between the KAs of the unlearned model and an ideal reference model on forget samples, while preserving KAs on retain samples. From this objective, we propose \textbf{Knowledge Association Reinforcement Unlearning (KARU)}, a unified framework for precise and efficient knowledge association unlearning,
which optimizes the KA objective from complementary gradient, distribution, and retention perspectives: (i) Knowledge Associated Token Reward (KATR), arising naturally from the gradient analysis of the objective and yielding fine-grained optimization signals; (ii) Logits Flattening (LF), which complements KATR by further promoting distribution-level forgetting; (iii) Adaptive Neighborhood Sampling (ANS), which improves retention by prioritizing retain samples more susceptible to collateral degradation during unlearning.

Our contributions are summarized as follows:

\begin{enumerate}
\item We introduce knowledge association as the fundamental forgetting target and formulate a unified KA-based optimization objective for LLM unlearning.

\item We derive complementary optimization mechanisms from the KA objective through complementary gradient and distribution perspectives, enabling fine-grained and distribution-aware knowledge removal.

\item We reveal the locality of collateral degradation and propose adaptive neighborhood sampling to improve the forget--retain trade-off.
\end{enumerate}

We evaluate KARU on representative unlearning benchmarks across diverse LLM backbones and RL algorithms. KARU consistently achieves stronger forgetting while maintaining comparable retention and general capabilities, improving the aggregate forget score by up to 30.5\% over the state-of-the-art (SOTA) baseline. Ablation studies further confirm the effectiveness of each component.

\section{Related Works}
\label{sec:related-works}

\subsection{LLM Unlearning}
LLM unlearning has emerged as a promising solution for removing the influence of undesirable training data, such as copyrighted content, private information, and harmful knowledge, while preserving the model's utility on unrelated tasks~\citep{jang2022knowledgeunlearningmitigatingprivacy, yao2024largelanguagemodelunlearning}.
Existing approaches to LLM unlearning can be broadly categorized into several paradigms. Gradient-based methods optimize against forget examples, often by directly opposing the original training objective~\citep{jang2022knowledgeunlearningmitigatingprivacy}. Preference optimization methods, such as NPO and SimNPO~\citep{zhang2024npo, fan2025simplicityprevailsrethinkingnegative}, formulate unlearning as preference learning between desired and undesired behaviors. Other approaches modify model parameters~\citep{ilharco2023taskvector,wuerkaixi2025adaptive}, edit factual knowledge directly~\citep{wu2023depndetectingeditingprivacy,meng2023locatingeditingfactualassociations}, or explicitly optimize predictive distributions~\citep{entesari2026constrained, zhai2026maximizing,di2024labelsmoothingimprovesmachine,Zhou_Qiang_Zade_Zytko_Khanduri_Zhu_2026}.
Among them, some recent studies have explored more fine-grained objectives for LLM unlearning. TNPO~\citep{wang2025rethinking} and TIF~\citep{Zhou_Qiang_Zade_Zytko_Khanduri_Zhu_2026} move beyond sequence-level optimization toward token-level interventions, while CEU~\citep{entesari2026constrained} highlights the limitations of suppressing target-token probabilities without explicitly accounting for the predictive distribution. Building on these insights, PALU~\citep{zhai2026maximizing} combines localized intervention on sensitive spans with entropy-based distribution flattening. However, extending such fine-grained objectives to diverse model-generated outputs without relying on predefined response-level or manually annotated span-level supervision remains challenging. KARU addresses this challenge by formulating forgetting around knowledge associations and jointly optimizing token-level and distribution-level objectives, enabling fine-grained unlearning across diverse model-generated outputs while preserving related knowledge.

\subsection{Reinforcement Learning for LLMs}
Reinforcement learning (RL) has become a widely adopted paradigm for aligning LLMs with desired behaviors through reward-driven optimization beyond supervised fine-tuning. Recent policy optimization algorithms~\citep{schulman2017proximalpolicyoptimizationalgorithms,shao2024deepseekmathpushinglimitsmathematical,hu2025reinforcestabilizingcriticfreepolicy,mopps, lpo}, have been extensively used for LLM post-training.
A key design choice in RL is reward modeling, where existing approaches mainly differ in supervision granularity. Outcome reward models (ORMs) evaluate only final responses, while process reward models (PRMs) provide intermediate supervision over reasoning steps or token-level decisions~\citep{lightman2023letsverifystepstep,ma2023let,lare}. Although finer-grained supervision enables better credit assignment, it typically requires additional reward models or expensive annotations.
Recently, RL has been explored for LLM unlearning due to its flexible reward-driven optimization and the potential for data-efficient adaptation~\citep{zhang2025rulereinforcementunlearningachieves,liu2024rethinkingmachineunlearninglarge}. However, existing RL-based unlearning methods predominantly rely on sequence-level rewards, which assign identical supervision signals to all tokens and fail to identify knowledge-bearing tokens.

\section{Preliminaries}
\label{sec:preliminaries}

\paragraph{Notation.}
Let $\D=\D_f\cup\D_r$ denote the pre-training corpus, where $\D_f$ and $\D_r$ represent the forget and retain sets respectively. For a query $x$ sampled from $\D$, the language model $\pi_\theta$ pretrained on $\D$ generates a response sequence $y=(y_1,\dots,y_T)$ autoregressively. Here, $y_t$ denotes the token generated at position $t$, $T$ denotes the response length, and $z_t=(x,y_{<t})$ denotes the context for predicting $y_t$. The token vocabulary is denoted by $\mathcal{V}$, where $v\in\mathcal{V}$ represents an individual token.
LLM unlearning aims to remove the impact of $\D_f$ while preserving the capabilities associated with $\D_r$. The ideal unlearned model is considered to be the retrained model $\pi_{\text{ref}}$ trained solely on $\D_r$.

\subsection{Problem Statement}
Most existing unlearning methods~\citep{chen2023unlearn,yao2024largelanguagemodelunlearning,li2025class} formulate unlearning as a supervised optimization problem:
\begin{equation}
\min_\theta \underbrace{\mathbb E_{(x_f,y_f)\sim\mathcal D_f}\bigl[\ell_f(y_f\mid x_f;\theta)\bigr]}_{\text{forget}}+\lambda \underbrace{\mathbb E_{(x_r,y_r)\sim\mathcal D_r} \bigl[\ell_r(y_r\mid x_r;\theta)\bigr]}_{\text{retain}}.
\label{eq:traditional_unlearning}
\end{equation}

While this formulation provides a general optimization framework, it treats all generated tokens uniformly and does not distinguish their different roles during unlearning.

Since current applications of LLM unlearning predominantly involve entity-centric factual knowledge, we focus on this forgetting scenario in this work.
Consider an example. Given the forget query ``What is the title of \textit{Stephen King}'s first published novel?'', the model may generate ``Stephen King's first published novel is \textit{Carrie}.'' Existing sequence-level objectives treat all tokens in the response uniformly, although only the token ``Carrie'' encodes the factual information to be removed. We refer to such tokens that directly express the target knowledge as \emph{knowledge tokens}, and the corresponding locations in the generated sequence as \emph{knowledge-token positions}. In contrast, the remaining linguistic tokens mainly reflect general language ability and should ideally be preserved. Therefore, precise unlearning requires identifying which parts of the generated sequence actually encode the knowledge to be removed.

To investigate this question, we analyze how unlearning changes the model's predictive distributions across different token positions. As shown in Fig.~\ref{fig:knowledge-analysis}, the distribution shift after unlearning highly concentrates on knowledge tokens, while normal tokens exhibit substantially smaller changes. The observation indicates that unlearning should focus on these knowledge tokens rather than optimizing uniformly over the entire generated sequence.

\begin{figure}[t]
    \centering
    \includegraphics[width=0.35\linewidth]{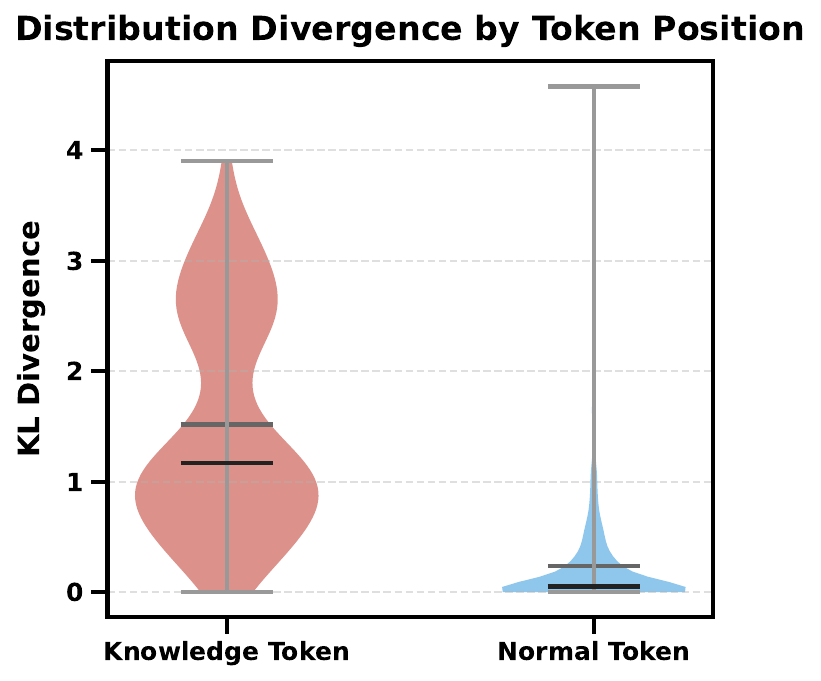}
    \includegraphics[width=0.35\linewidth]{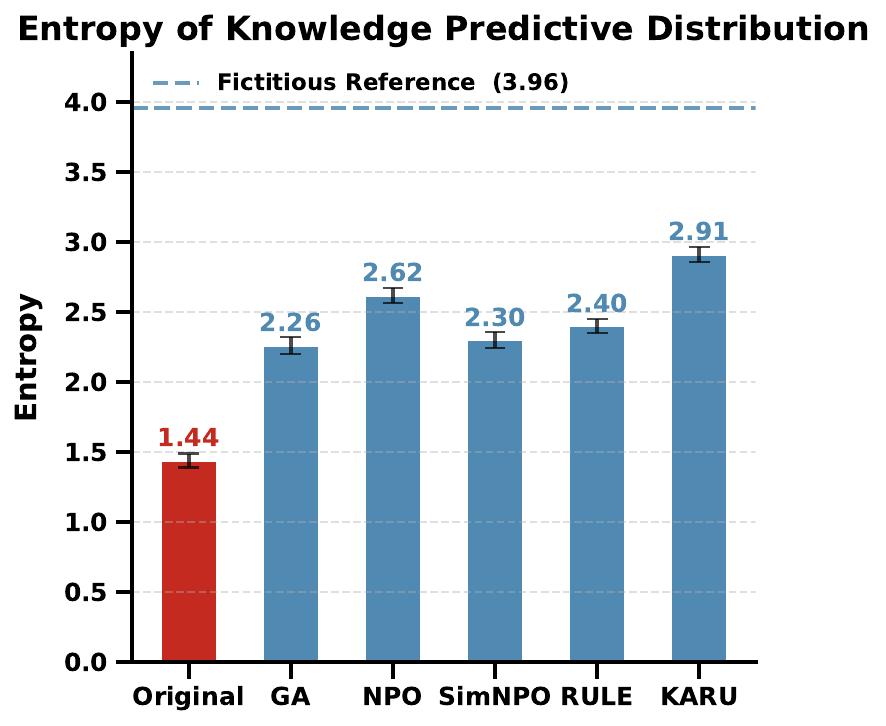}
    \caption{
    Empirical analysis of predictive distribution changes during unlearning.
    \textbf{Left}: Distribution shifts concentrate on knowledge tokens.
    \textbf{Right}: Comparison of entropy at knowledge-token positions across different unlearning models and a fictitious unseen-knowledge reference.
    }
    \label{fig:knowledge-analysis}
\end{figure}

\paragraph{Knowledge Association.}
Motivated by the observation above, we characterize the forgetting target through the model's predictive distribution at knowledge-token positions. We term this conditional predictive distribution, which associates a query context with its target knowledge, the knowledge association (KA).

\begin{definition}[Knowledge Association]
\label{def:knowledge_assoc}
For a query $x$ with model response $y$, let $t_0$ be the first knowledge-token position of the target knowledge in $y$ and $z_{t_0}=(x,y_{<t_0})$ the corresponding prediction context. We represent the knowledge association by:
\begin{equation}
\K_\theta(z_{t_0}) \triangleq \pi_\theta(\cdot \mid z_{t_0}).
\label{eq:knowledge_assoc}
\end{equation}
\end{definition}

Treating KA as the fundamental forgetting target, we reformulate LLM unlearning as suppressing knowledge associations.
Let $\K_{\text{ref}}$ denote the knowledge association of a reference model, ideally the retrained model $\pi_{\text{ref}}$. We then minimize the discrepancy between $\K_\theta(z_t)$ and $\K_{\text{ref}}(z_t)$ on forget samples while preserving predictive capability on retain samples:
\begin{equation}
\begin{aligned}
\min_\theta\;\; & \E_{x\sim\D_f, y\sim\pi_\theta(\cdot\vert x)}\left[d\bigl(\K_\theta(z_{t_0}),\,\K_{\text{ref}}(z_{t_0})\bigr)\right] \\
\text{s.t.}\;\; & \E_{(x',y')\sim\D_r}\left[\phi_{y'}(\K_\theta(z_{t_0'}))\right] \geq \tau,
\end{aligned}
\label{eq:unified_obj}
\end{equation}
where $d(\cdot,\cdot)$ is a discrepancy measure, and $\phi_{y'}(\K_\theta(z_{t_0'}))=\pi_\theta (y'_{t_0'}\mid x',y'_{<t_0'})$ extracts the probability assigned to the knowledge token.
Eq.~\eqref{eq:unified_obj} defines the ideal objective but is intractable in practice due to the unavailable reference association $\K_{\mathrm{ref}}$.
We therefore derive a practical RL framework that approximately optimizes the objective.

\section{Knowledge Association Reinforcement Unlearning}

\begin{figure*}[t]
  \centering
  \includegraphics[width=0.9\linewidth]{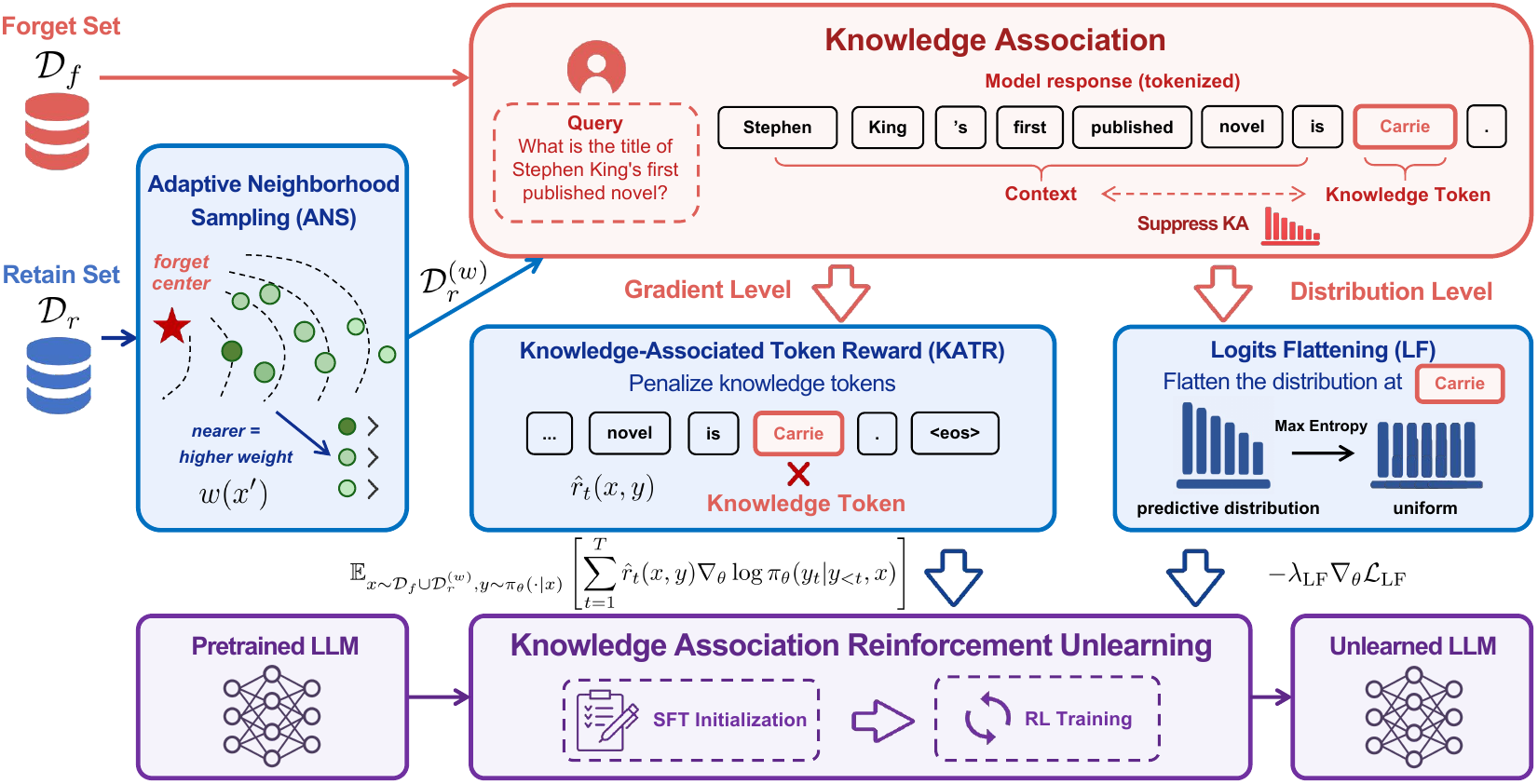}
  \caption{\textbf{Overview of KARU pipeline.} Starting from KA as the forgetting target, KARU derives complementary optimization mechanisms for forgetting and retention: (i) KATR provides fine-grained forgetting supervision, (ii) LF promotes distribution-level knowledge removal, and (iii) ANS strengthens retention on vulnerable retain samples.}
  \label{fig:pipeline}
\end{figure*}

To optimize the knowledge association unlearning objective within the RL framework, we propose KARU, which consists of three complementary components. For forgetting, gradient analysis derives a token-level reward for fine-grained credit assignment, while logits flattening enforces distribution-level suppression of forgotten knowledge. For retention, adaptive proximity-weighted sampling refines the uniform retain supervision by prioritizing vulnerable neighbors and mitigating collateral forgetting.

\subsection{Knowledge Associated Token Reward}
\label{subsec:token-level}

\subsubsection{Limitations of Sequence-Level Rewards}
Existing RL-based unlearning methods~\citep{zhang2025rulereinforcementunlearningachieves} assign a single scalar reward $r^{\mathrm{seq}}(x,y)$ to the entire response. Specifically, given an input query $x\in \D$, the RL objective is defined as
\begin{equation}
    J(\theta)=\E_{y \sim \pi_\theta(\cdot \mid x)}[r^{\mathrm{seq}}(x,y)],
\label{eq:rl_objective}
\end{equation}
whose policy gradient is
\begin{equation}
\nabla_\theta J(\theta) = \E_{y \sim \pi_\theta (\cdot \mid x)}\Bigl[r^{\mathrm{seq}}(x,y) \sum_{t=1}^{T} \nabla_\theta \log\pi_\theta(y_t\mid x,y_{<t})\Bigr].
\label{eq:pg_surrogate}
\end{equation}
Therefore, the same sequence-level reward is propagated to all token positions, without distinguishing their individual contributions. Table~\ref{tab:knowledge_ratio} in Appendix~\ref{app:knowledge-analysis} empirically shows that knowledge tokens usually account for only a small fraction of the generated responses, suggesting that the forgetting signal may be diluted across irrelevant tokens and motivating finer-grained credit assignment.
To address this, we next show that token-level rewards naturally arise from the gradient structure of the objective in Eq.~\eqref{eq:unified_obj}.

\subsubsection{Gradient Derivation of Token-Level Rewards}
For a forget sample $x\in \D_f$ and any token position $t$, let $\pi_\theta^t(v)=\pi_\theta(v\mid z_t)$ and $\pi_{\mathrm{ref}}^t(v)=\pi_{\mathrm{ref}}(v\mid z_t)$.
Instantiating the discrepancy measure $d(\cdot,\cdot)$ as the KL divergence, the local contribution of position $t$ to the forget objective in Eq.~\eqref{eq:unified_obj} is:
\begin{equation}
\mathcal{L}_f^t(\theta) \triangleq D_{\text{KL}}\bigl(\pi_\theta^t \;\|\; \pi_{\text{ref}}^t\bigr) = \sum_{v \in \mathcal{V}} \pi_\theta^t(v) \log \frac{\pi_\theta^t(v)}{\pi_{\text{ref}}^t(v)}.
\label{eq:Lf_kl}
\end{equation}
Taking the gradient gives:
\begin{equation}
\begin{aligned}
\nabla_\theta \mathcal{L}_f^t(\theta)\, = \mathbb{E}_{v \sim \pi_\theta^t} \left[ \log \frac{\pi_\theta^t(v)}{\pi_{\text{ref}}^t(v)} \cdot \nabla_\theta \log \pi_\theta^t(v) \right].
\end{aligned}
\label{eq:Lf_grad}
\end{equation}

Notably, Eq.~\eqref{eq:Lf_grad} shares the same structure with the token-level policy gradient $\mathbb{E}_{v \sim \pi_\theta^t} [r_{t,v} \nabla_\theta \log \pi_\theta^t(v)]$, where $r_{t,v}$ denotes the token-level reward assigned to token $v$ at position $t$. This suggests that the logarithmic coefficient can be naturally interpreted as a token-level reward signal. This motivates the following ideal token-level reward:
\begin{equation}
r_{t,v}^* = -\log{\frac{\pi_\theta(v\mid z_t)}{\pi_{\text{ref}}(v\mid z_t)}},
\label{eq:ideal_reward}
\end{equation}

\subsubsection{Reward Design}
Eq.~\eqref{eq:ideal_reward} shows that the ideal token-level reward is determined by the probability shift relative to the reference distribution.
We propose Knowledge Associated Token Reward (KATR) to approximate it. Guided by our KA formulation and the observed concentration of knowledge token distribution shifts, we focus reward assignment on these positions. Since the ideal reference model should lack of explicit memorization of the forgotten knowledge, it assigns weaker probability preference to these knowledge tokens compared with the original model. Consequently, $\pi_{\text{ref}}<\pi_\theta$ naturally induces negative reward ($-1$).

Additionally, prior works~\citep{jin2024rwkubenchmarkingrealworldknowledge,maini2024tofu,zhang2025rulereinforcementunlearningachieves} commonly encourage refusal as the desired behavior for forget queries. This effectively implies an additional high-probability mass on refusal tokens in the reference distribution $\pi_{\text{ref}}>\pi_\theta$, inducing positive reward ($+1$) for those refusal-related tokens.

For retain samples, KATR follows the same principle in the opposite direction by rewarding knowledge tokens and penalizing refusal tokens to preserve normal responses~\citep{zhang2025rulereinforcementunlearningachieves}.
The binary reward thus acts as a sign-level approximation of Eq.~\eqref{eq:ideal_reward}: it preserves the direction of the required probability shift while avoiding magnitude estimation on the unavailable reference model.
Appendix~\ref{app:token-reward-impl} provides the full specification and verifies that KATR concentrates supervision on knowledge tokens more effectively than sequence-level rewards.

\subsection{Logits Flattening}
\label{subsec:logits-flattening}

KATR approximates the ideal token-level reward from the gradient perspective, aligning the policy gradient with the objective in Eq.~\eqref{eq:Lf_kl} and providing a first-order optimization signal.
Complementarily, it does not explicitly constrain the predictive distribution, allowing probability mass to shift among related alternatives while remaining concentrated~\citep{jia2024soulunlockingpowersecondorder,zhai2026maximizing}.
To complement this, we further propose Logits Flattening~(LF) to approximately optimize Eq.~\eqref{eq:Lf_kl} from the distribution level.

\subsubsection{Distributional Forgetting Behavior}
Ideally, Eq.~\eqref{eq:Lf_kl} requires the reference distribution $\K_{\text{ref}}$, which is unavailable in practice.
We therefore approximate it using a fictitious knowledge set of synthetic query-response pairs.
These samples represent knowledge that the model has never acquired, serving as a practical proxy for the predictive behavior after
unlearning.
As shown in Fig.~\ref{fig:knowledge-analysis} (right), the original model exhibits substantially higher entropy on fictitious knowledge than on memorized knowledge, indicating that the absence of knowledge naturally leads to less concentrated predictions.
This observation suggests that effective knowledge unlearning should move knowledge-token predictions toward the higher-entropy behavior exhibited by fictitious knowledge.
This trend is also consistent with existing unlearning methods, which typically increase the entropy of forgotten knowledge predictions relative to the original model.

\subsubsection{Entropy as a Distributional Surrogate}

Based on this observation, we introduce the following surrogate principle:

\begin{principle}[Entropy Surrogate]
For knowledge tokens, the predictive distribution of an effectively forgotten model is characterized by higher entropy than that of a memorized model. Therefore, entropy provides a reference-free surrogate for approximating the ideal forgetting behavior.
\end{principle}

Following this principle, we propose Logits Flattening to approximate distribution-level forgetting by maximizing the entropy of knowledge-token predictions:
\begin{equation}
\loss_{\text{LF}} = -\E_{x\sim\D_f, y \sim \pi_\theta(\cdot \mid x)}\left[
\sum_{t=1}^{T}
\mathbb{I}(y_t\in\T_{\text{know}})
H\left(\pi_\theta(\cdot\mid x,y_{<t})\right)
\right],
\label{eq:lf_loss}
\end{equation}
where $H$ denotes the entropy of the predictive distribution, and $\mathbb{I}(y_t\in\T_{\text{know}})$ is an indicator function that selects only knowledge tokens for entropy maximization. Here, $\T_{\text{know}}$ denotes the pre-extracted set of knowledge tokens identified from $\D_f$, with the extraction procedure detailed in Appendix~\ref{app:token-reward-impl}.

\begin{wrapfigure}{r}{0.45\linewidth}
    \vspace{-18pt}
    \centering
    \includegraphics[width=\linewidth]{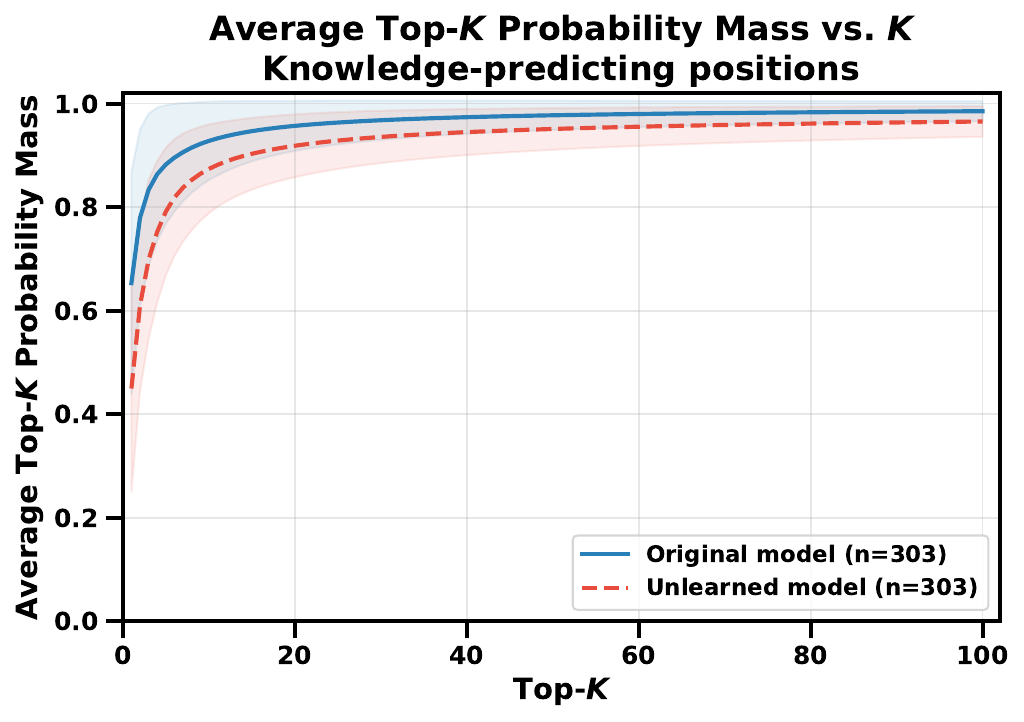}
    \vspace{-20pt}
    \caption{\textbf{Average Top-$k$ probability mass at knowledge-token positions.} The cumulative probability mass captured by Top-$k$ tokens increased rapidly with $k$ for both original and unlearned models.}
    \label{fig:tail_mass_cdf}
    \vspace{-15pt}
\end{wrapfigure}
However, computing the entropy over the full vocabulary is computationally expensive. In practice, we approximate the entropy in Eq.~\eqref{eq:lf_loss} using only the Top-$k$ probability mass. Under a mild tail-concentration assumption~\ref{ass:topk}, the approximation error remains controlled compared with the typical entropy range, while substantially reducing both computational and memory overhead.
As shown in Fig.~\ref{fig:tail_mass_cdf}, both the original and unlearned models concentrate most probability mass within a small set of Top-$k$ tokens at knowledge-token positions, so a moderate $k$ already captures nearly the entire predictive mass and keeps the approximation tight.

Maximizing entropy promotes uncertainty in predictions at relevant knowledge-token positions, approximating the behavior of a model that has not acquired the target knowledge~\citep{entesari2026constrained,zhai2026maximizing}. Comparison of distributional surrogates in Appendix~\ref{sec:analysis-of-lf} further shows that entropy most consistently distinguishes memorized from fictitious knowledge compared with other distribution-level metrics.

\subsection{Adaptive Neighborhood Sampling}
\label{subsec:adaptive-neighbor}

\subsubsection{Locality of Collateral Degradation}
\begin{wrapfigure}{r}{0.45\linewidth}
    \centering
    \vspace{-18pt}
    \includegraphics[width=\linewidth]{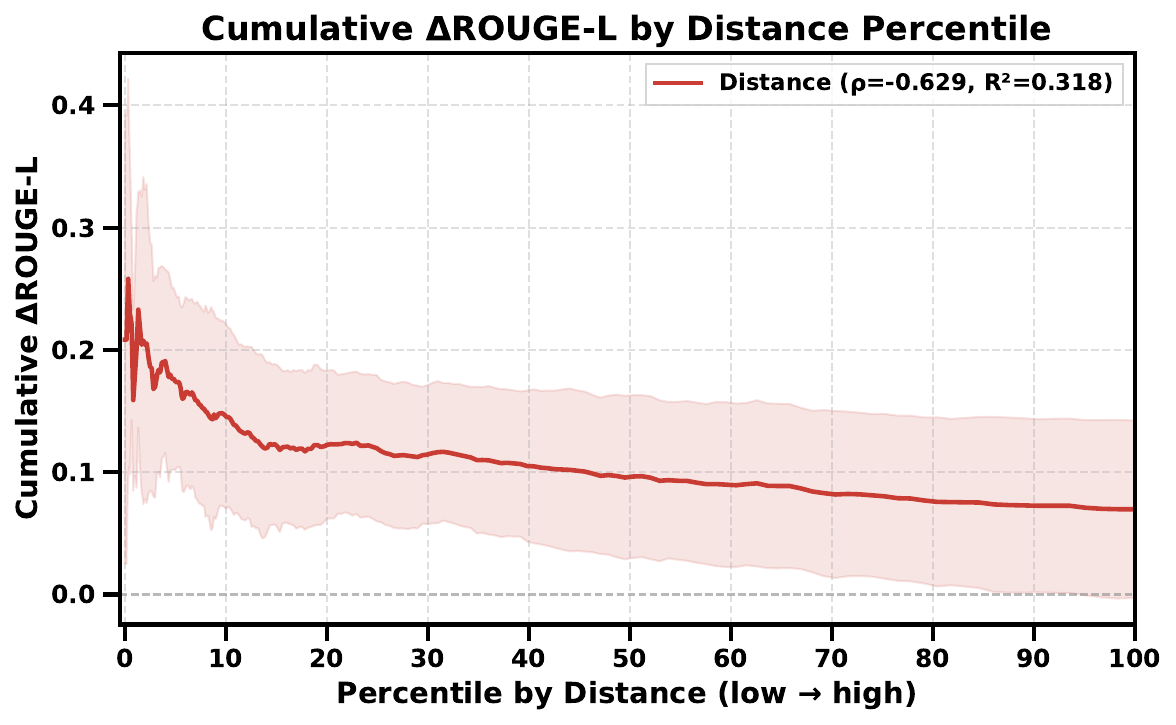}
    \vspace{-5pt}
    \caption{\textbf{Collateral degradation versus representation-space proximity.} Retain samples closer to forget set exhibit larger ROUGE-L degradation after unlearning.}
    \vspace{-10pt}
    \label{fig:proximity-effect}
\end{wrapfigure}
The retain constraint in Eq.~\eqref{eq:unified_obj} aims to preserve retained knowledge.
However, prior works~\citep{dolgova2026forgetting,ha2025unlearning,cheng2026machine} show that forgetting inevitably degrades some retain samples, a phenomenon referred to as collateral degradation, and that such degradation is concentrated on a subset of retain samples.
Theoretically, \citet{dolgova2026forgetting} show that collateral forgetting is related to the similarity between retain samples and the forget subspace.

Empirically, we measure the representation-space distance between retain samples and the forget set and examine their corresponding performance degradation after unlearning. Performance degradation is measured by the decrease in ROUGE-L similarity~\citep{lin-2004-rouge} between the unlearned and original model outputs.
Consistent with the theoretical prediction, Fig.~\ref{fig:proximity-effect} shows that collateral degradation exhibits a clear proximity effect: retain samples closer to the forget set in representation space suffer larger performance degradation after unlearning.
Accordingly, we propose Adaptive Neighborhood Sampling (ANS), which strengthens retention supervision by assigning larger sampling weights to nearby retain samples.

\subsubsection{Proximity Reweighting}
Specifically, let $rank(x')$ denote the rank of a retain sample $x'$ obtained by sorting all retain samples according to ascending $\operatorname{dist}(x',\D_f)$, where smaller values indicate closer samples.
For a percentile fraction $\alpha \in (0,1]$, only the closest $\lceil \alpha |\D_r| \rceil$ retain samples are assigned larger sampling weights, with the sampling weight defined as:
\begin{equation}
w(x') = 1 + \beta \cdot \max\Bigl(0,\; 1 - \frac{rank(x')}{\lceil \alpha |\D_r| \rceil}\Bigr),
\label{eq:weighted_dist}
\end{equation}
where $\beta>0$ controls the maximum reweighting strength.
The weighting decreases linearly with the proximity rank, while rank-based weighting avoids sensitivity to the absolute scale of the distance metric.
Moreover, since the neighborhood size and weights depend only on rank orderings rather than absolute distances, the scheme is nonparametric with respect to the density of retain samples in representation space and adapts to arbitrary sample distributions.
During training, retain samples are drawn with probability proportional to
$w(x')$, resulting in an effective reweighted retain distribution
$\D_r^{(w)}$. The concrete distance instantiation is detailed in Appendix~\ref{app:analysis-of-ans}.
The proximity computation incurs only a one-time offline cost, requiring a single forward pass over $\D_r$, and the measured per-step overhead is negligible as reported in Appendix~\ref{app:complexity}.

Although ANS is constructed offline before RL training, representation drift during optimization may in principle distort the proximity estimates.
We therefore track the distances of retain samples throughout training on a representative RWKU subset.
The proximity structure remains nearly unchanged: 95\% of retain samples exhibit less than 2.8\% relative distance variation, with a mean change of 1.1\%, and the induced ranking achieves a Spearman correlation of 0.998 across training steps.
These results indicate that the representation-space proximity used by ANS is sufficiently stable, supporting offline computation without repeated distance updates during RL optimization.

\subsection{Overall Implementation}
\label{subsec:complete-framework}

The overall optimization update of KARU is summarized below:
\begin{equation}
\nabla_\theta J_{\text{KARU}} = \E_{x\sim\D_f\cup\D_r^{(w)}, y\sim\pi_\theta(\cdot|x)} \left[ \sum_{t=1}^{T} \hat{r}_t \nabla \log\pi_\theta(y_t\mid x,y_{<t}) \right] -\lambda_{\mathrm{LF}} \nabla_\theta\mathcal{L}_{\mathrm{LF}},
\label{eq:total_objective}
\end{equation}
where $\hat{r}_t$ is the token-level reward produced by KATR, $\D_r^{(w)}$ is the retain distribution reweighted by ANS, and $\lambda_{\mathrm{LF}}$ controls the strength of the logits flattening regularization.

The resulting framework integrates fine-grained optimization, distribution-level forgetting, and proximity-aware retain reweighting.
Moreover, KARU is agnostic to the underlying policy optimization algorithm and can be seamlessly integrated into existing RL-based unlearning methods with negligible computational overhead, as discussed in Appendix~\ref{app:complexity}.
The overall training procedure is summarized in Algorithm~\ref{alg:rl-unlearning}.

\begin{algorithm}[t]
\caption{Knowledge Association Reinforcement Unlearning (KARU)}
\label{alg:rl-unlearning}
\begin{algorithmic}[1]
\REQUIRE Original model $\pi_{\theta_0}$, forget set $\D_f$, retain set $\D_r$, hyperparameters $\lambda_{\text{LF}}, k, \alpha_{\mathrm{WS}}, \beta_{\mathrm{WS}}$ 
\ENSURE Unlearned model $\pi_{\text{unlearn}}$

\STATE Rejection-Initialize $\pi_\theta$ from $\pi_{\theta_0}$ using $\D_f$ via SFT
\STATE Extract knowledge token sets $\T_{\text{know}}$ from $\D_f$ via NER
\STATE Compute representation-space distances $d(x', \D_f)$ for all $x' \in \D_r$
\STATE Compute proximity weights $\{w(x')\}_{x' \in \D_r}$ (Eq.~\eqref{eq:weighted_dist}) and construct the weighted retain set $\D_r^{(w)}$ by sampling from $\D_r$ with probability proportional to $w(x')$

\FOR{epoch = 1 to $N_{\text{epochs}}$}
    \FOR{batch $(x, y) \in \D_f \cup \D_r^{(w)}$}
        \IF{$(x, y) \in \D_f$}
            \STATE Compute KATR $\{r_t\}_{t=1}^T$ for forget tokens (Eq.~\eqref{eq:token_reward})
            \STATE Compute Logits Flattening regularization $\loss_{\text{LF}}$ (Eq.~\eqref{eq:lf_loss})
        \ELSE
            \STATE Compute KATR $\{r_t\}_{t=1}^T$ for retain tokens (Eq.~\eqref{eq:token_reward})
            \STATE Set $\loss_{\text{LF}} = 0$
        \ENDIF
        \STATE Compute KARU objective $J_{\mathrm{KARU}}$ (Eq.~\eqref{eq:total_objective})
        \STATE Update model parameters $\theta \leftarrow \theta + \eta\nabla_\theta J_{\mathrm{KARU}}$
    \ENDFOR
\ENDFOR

\RETURN $\pi_{\text{unlearn}} = \pi_\theta$
\end{algorithmic}
\end{algorithm}

\section{Experiments}
\label{sec:experiments}

\paragraph{Setups.} We evaluate the KARU framework on two representative benchmarks: RWKU~\citep{jin2024rwkubenchmarkingrealworldknowledge} and MUSE~\citep{shi2024muse}. Following RULE~\citep{zhang2025rulereinforcementunlearningachieves}, we use Llama3-8B-Instruct~\citep{grattafiori2024llama3herdmodels} as the primary base model on RWKU, with Llama3.2-3B-Instruct and Qwen2.5-7B-Instruct~\citep{qwen2025qwen25technicalreport} for cross-model validation; MUSE experiments use ICLM-7B~\citep{shi2024incontextpretraining}. Our RL-based unlearning is initialized from the rejection-tuned model of RULE and uses PPO and GRPO as the optimization algorithms. Further implementation details are provided in Appendix~\ref{app:training-details}.

\paragraph{Baselines.} We compare against representative baseline methods spanning SFT-based and RL-based paradigms:
\begin{itemize}
    \item \textbf{GA}~\citep{maini2024tofu}: Gradient Ascent on the forget set loss;
    \item \textbf{NPO}~\citep{zhang2024npo}: Negative Preference Optimization, reframing unlearning as minimizing preference toward undesired outputs;
    \item \textbf{SimNPO}~\citep{meng2024simposimplepreferenceoptimization}: An improved variant of NPO that removes reference model dependency;
    \item \textbf{PALU}~\citep{zhai2026maximizing}: A localized unlearning method that maximizes entropy by flattening the top-$K$ logits at the first $N$ tokens of each sensitive span;
    \item \textbf{RT}~(Rejection Tuning): Supervised fine-tuning on refusal responses for forget queries;
    \item \textbf{RULE}~\citep{zhang2025rulereinforcementunlearningachieves}: The state-of-the-art RL-based unlearning method with response-level reward.
\end{itemize}
For GA, NPO, and SimNPO, we additionally evaluate variants augmented with Gradient Difference Regularization (+GDR) and KL Regularization (+KLR) on the retain set~\citep{shi2024muse}.

\paragraph{Evaluation Metrics.}
Following RWKU~\citep{jin2024rwkubenchmarkingrealworldknowledge}, we evaluate unlearning from three perspectives: \textbf{forget efficacy}, \textbf{retain locality}, and \textbf{general capability}. Forget efficacy is assessed using three probes: fill-in-the-blank (FB) for knowledge memorization, question answering (QA) for knowledge manipulation, and adversarial attacks (AA) that attempt to elicit forgotten knowledge. Retain locality is evaluated using neighborhood FB and QA probes. All forget and retain probes are scored by ROUGE-L recall~\citep{lin-2004-rouge}; the aggregate forget score averages FB, QA, and AA (lower is better), while the aggregate retain score averages the neighborhood FB and QA (higher is better). General capability is evaluated along three axes: \textbf{General Knowledge (Gen)}, measured by 5-shot accuracy on MMLU~\citep{hendrycks2021measuringmassivemultitasklanguage}; \textbf{Reasoning (Rea)}, exact-match accuracy on BBH~\citep{suzgun2022challengingbigbenchtaskschainofthought}; and \textbf{Fluency (Flu)}, the weighted bi- and tri-gram entropy of AlpacaEval~\citep{li2023alpacaeval} outputs. Formal definitions of all metrics follow RWKU~\citep{jin2024rwkubenchmarkingrealworldknowledge}.

\begin{figure}[t]
  \centering
  \includegraphics[width=0.7\linewidth]{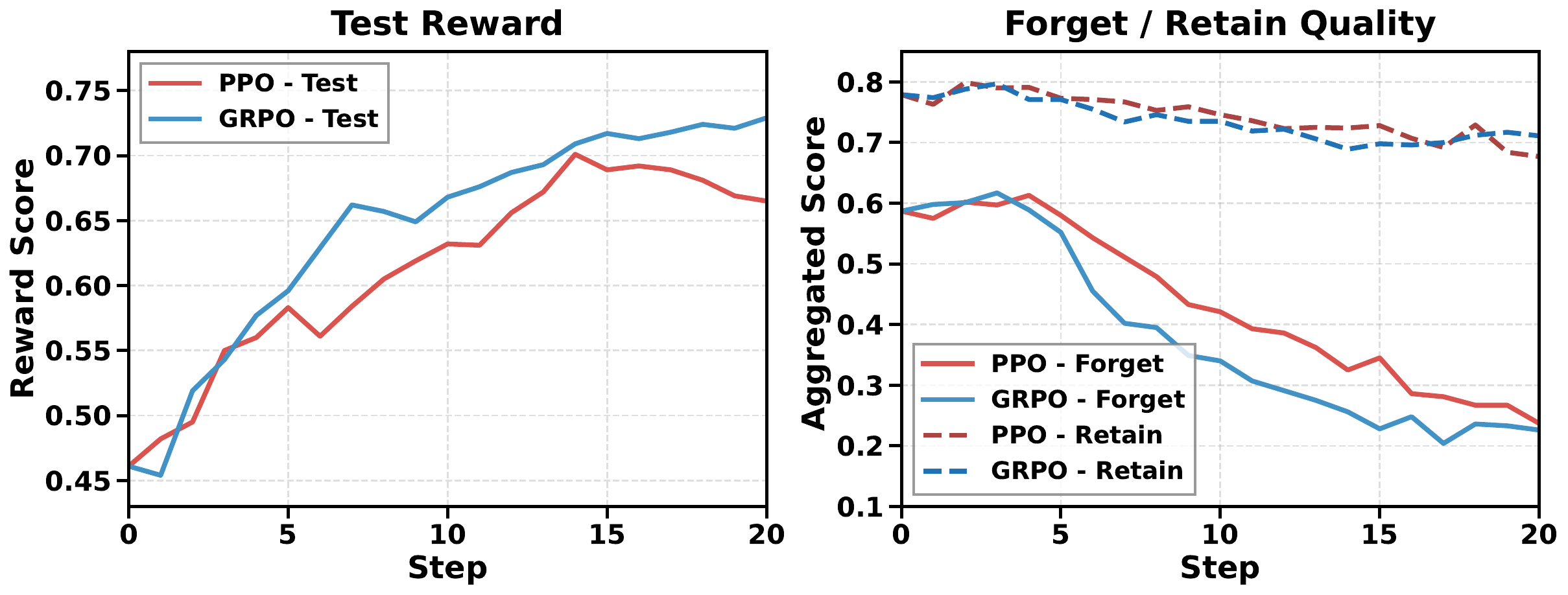}
  \caption{\textbf{Left:} Test reward curves under PPO and GRPO. \textbf{Right:} Forget and retain metric trajectories during training.}
  \vspace{-5pt}
  \label{fig:training-curves}
\end{figure}

\subsection{Main Results}

Table~\ref{tab:main-results} presents the comprehensive experimental results on the RWKU benchmark using Llama3-8B-Instruct as the backbone model. Fig.~\ref{fig:training-curves} provides a detailed analysis of KARU's training dynamics under different RL algorithms, showing the testing reward trajectories along with the corresponding changes in the overall forget and retain performance. Key findings are summarized below.

\begin{wraptable}{r}{0.6\linewidth}
\centering
\scriptsize
\setlength{\tabcolsep}{2pt}
\renewcommand{\arraystretch}{1.1}
\vspace{-10pt}
\caption{Main results on the RWKU benchmark. Best in \textbf{bold}, second best \underline{underlined}. FB: Fill-in-the-Blank; QA: Question-Answer; AA: Adversarial Attack; Gen: General Knowledge Ability; Rea: Reasoning Ability; Flu: Fluency. $\downarrow$/$\uparrow$: lower/higher is better.}
\label{tab:main-results}
\adjustbox{width=\linewidth}{
\begin{tabular}{l|cccc|ccc|ccc}
\toprule
\multirow{2}{*}{\textbf{Method}} & \multicolumn{4}{c}{Forget ($\downarrow$)} & \multicolumn{3}{c}{Retain ($\uparrow$)} & \multicolumn{3}{c}{General ($\uparrow$)} \\
\cmidrule(lr){2-5}\cmidrule(lr){6-8}\cmidrule(lr){9-11}
& FB & QA & AA & All & FB & QA & All & Gen & Rea & Flu \\
\midrule
Original & 85.6 & 70.3 & 74.7 & 76.9 & 93.1 & 82.0 & 87.6 & 63.3 & 40.5 & 706.6 \\
\midrule
GA         & 72.0 & 64.6 & 68.5 & 68.4 & \underline{85.0} & \underline{74.7} & \underline{79.8} & \underline{63.5} & 38.8 & 707.7 \\
GA$_{\text{GDR}}$ & 72.6 & 64.0 & 69.7 & 68.8 & \textbf{86.2} & \textbf{76.5} & \textbf{81.4} & \textbf{63.6} & 40.5 & 707.1 \\
GA$_{\text{KLR}}$ & 70.7 & 57.5 & 69.9 & 66.1 & 80.5 & 70.5 & 75.5 & 63.4 & 40.1 & 706.5 \\
\midrule
NPO        & 38.7 & 22.9 & \underline{33.0} & 31.5 & 76.7 & 67.5 & 72.1 & 62.8 & 37.9 & 693.1 \\
NPO$_{\text{GDR}}$ & 39.7 & 26.2 & 35.6 & 33.9 & 76.9 & 69.9 & 73.4 & 63.2 & 38.4 & 703.6 \\
NPO$_{\text{KLR}}$ & 39.8 & 27.4 & 35.6 & 34.3 & 77.4 & 70.1 & 73.7 & 63.3 & 39.9 & 699.6 \\
\midrule
SimNPO & 45.7 & 38.3 & 50.2 & 44.7 & 82.2 & 71.8 & 77.0 & 62.7 & 38.5 & 700.2 \\
SimNPO$_{\text{GDR}}$ & 53.7 & 47.7 & 54.6 & 52.0 & 83.6 & 73.0 & 78.3 & 63.4 & 39.1 & \underline{708.5} \\
SimNPO$_{\text{KLR}}$ & 46.1 & 36.7 & 52.2 & 45.0 & 82.7 & 73.3 & 78.0 & 63.3 & 39.1 & 705.9 \\
\midrule
PALU & 45.8 & 32.7 & 35.6 & 38.0 & 79.4 & 73.8 & 76.6 & 63.0 & 35.6 & 703.6 \\
\midrule
RT & 75.2 & 47.1 & 51.6 & 58.7 & 85.9 & 70.0 & 77.9 & \textbf{63.6} & 38.5 & 680.7 \\
\midrule
RULE$_{\text{PPO}}$  & 34.5 & 22.2 & 34.2 & 30.3 & 70.4 & 52.4 & 61.4 & 63.3 & 38.9 & 641.5 \\
RULE$_{\text{GRPO}}$ & 32.7 & 23.7 & 41.1 & 32.5 & 77.0 & 66.9 & 71.9 & 63.3 & \underline{41.0} & 702.0 \\
\midrule
KARU$_{\text{PPO}}$  & \underline{25.1} & \underline{13.5} & \textbf{32.5} & \underline{23.7} & 71.9 & 63.6 & 67.7 & \underline{63.5} & \underline{41.0} & \textbf{717.7} \\
KARU$_{\text{GRPO}}$ & \textbf{22.5} & \textbf{11.2} & 34.0 & \textbf{22.6} & 76.5 & 65.7 & 71.1 & 63.0 & \textbf{42.1} & 707.7 \\
\bottomrule
\end{tabular}
}
\vspace{-20pt}
\end{wraptable}
\paragraph{KARU improves forgetting substantially.}
KARU consistently achieves the strongest forgetting performance among all evaluated methods.
Compared with the SOTA RL-based method RULE, KARU reduces the aggregate forget score by 6.6 and 9.9 points under PPO and GRPO, respectively, corresponding to relative reductions of 21.8\% and 30.5\%.
The improvements are especially notable on the FB and QA probes, which require recovering the target knowledge from the given context and thus directly test the association that KA characterizes as the forgetting target. This validates that KARU effectively reduces the predictive association with forgotten knowledge.

\paragraph{KARU preserves retention and general capabilities.}
Beyond significantly improved forgetting performance, KARU maintains comparable retention performance across different RL algorithms.
Compared with RULE under the same algorithm, KARU$_{\text{GRPO}}$ decreases the overall retain score by only 1.1\%, while KARU$_{\text{PPO}}$ even achieves improved retention, demonstrating that stronger forgetting does not compromise retained knowledge.
We note that a decline in absolute retention relative to the original model is inherent to unlearning, where forgetting and retention are in tension, and our goal is not to maximize retention but to attain the best forget--retain trade-off while maintaining sufficient retention, as formulated by the constraint in Eq.~\eqref{eq:unified_obj}.
Furthermore, KARU achieves competitive performance on general capability evaluation: KARU$_{\text{GRPO}}$ achieves the highest Reasoning score, while KARU$_{\text{PPO}}$ obtains the best Fluency among all evaluated methods, indicating that logits flattening suppresses forgotten KAs without degenerating outputs into ungrammatical text; the Gen scores of all methods remain at a comparable level.
These results demonstrate that KARU achieves effective knowledge removal without sacrificing the overall utility.

\begin{figure}[t]
  \centering
  \includegraphics[width=0.45\linewidth]{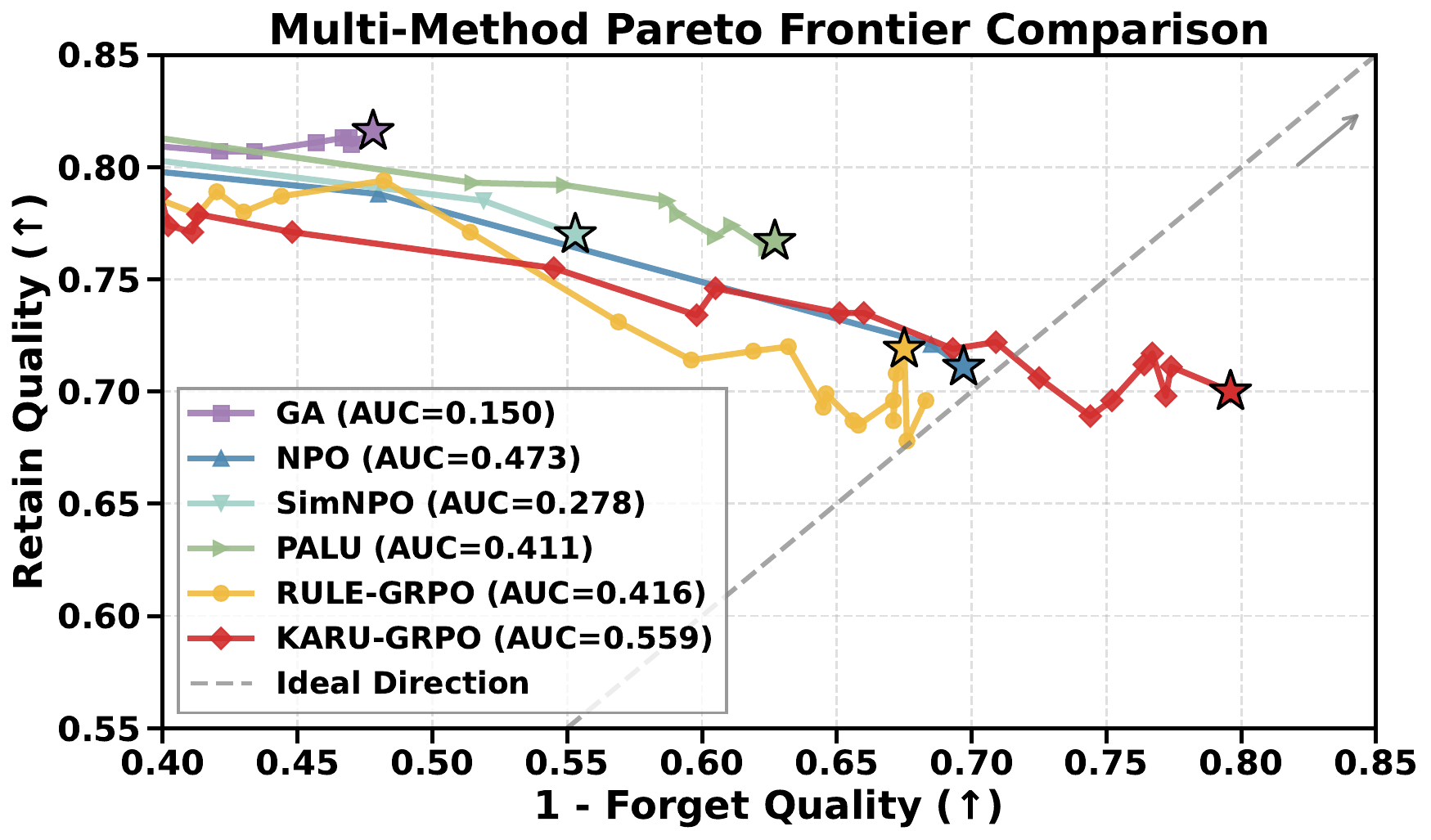}
  \includegraphics[width=0.45\linewidth]{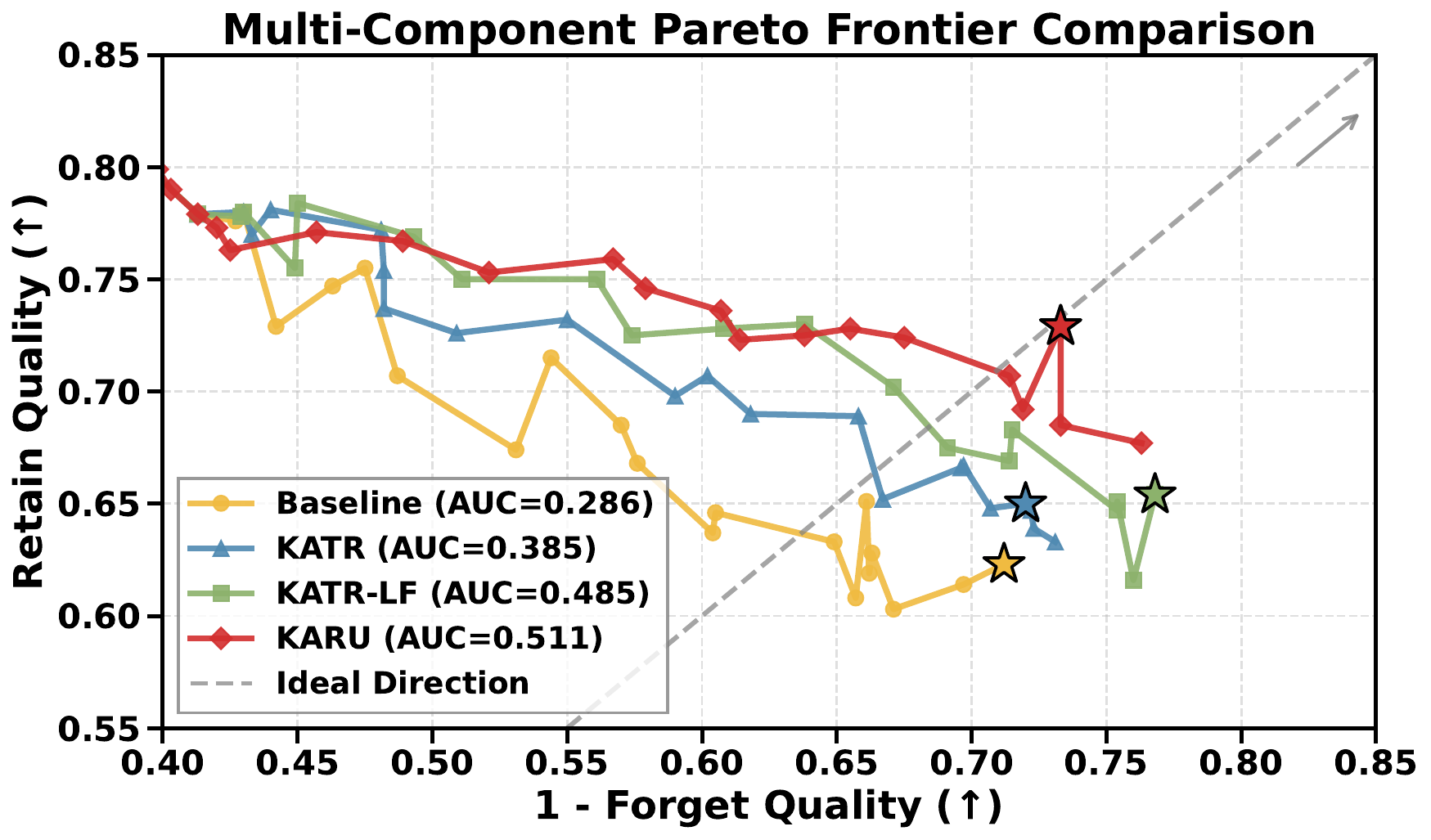}
  \caption{\textbf{Pareto frontier analysis.} The \textbf{top} panel compares forget--retain trade-offs across unlearning methods, with KARU achieving the most favorable frontier. The \textbf{bottom} panel shows progressive frontier improvements as incremental components integration under PPO. Curves trace evaluation checkpoints sorted by $1-\mathrm{FQ}$; AUC is computed by trapezoidal integration over the clipped, normalized shared range, with higher values indicating a better trade-off.}
  \label{fig:pareto-analysis}
\end{figure}

\paragraph{KARU achieves the best forget--retain trade-off.}
Fig.~\ref{fig:pareto-analysis} (top) visualizes the empirical trade-off between aggregated forgetting and retain quality. For reproducibility, the plotted score is $1-\mathrm{FQ}$ so that both axes are higher-is-better, and each curve traces evaluation checkpoints sorted by this score. We compute AUC by clipping all methods to the same feasible domain shown in the plot, min--max normalizing the two axes over that domain, and applying trapezoidal integration. Within the SFT-based family, GA and its variants remain limited in forgetting improvement; entropy-based distribution-level methods such as PALU concentrate supervision on fine-grained sensitive text spans and thereby achieve markedly better forgetting at comparable retention; and preference optimization-based approaches like NPO attain still stronger forgetting but often pay a larger retention cost. In contrast, RL-based methods push this trade-off toward a more favorable frontier: KARU achieves stronger forgetting while maintaining competitive retention, attains the highest AUC among all evaluated methods, and progressively expands the empirical Pareto frontier during RL optimization. This demonstrates the effectiveness of our KA-guided framework in optimizing the forget--retain balance.

\begin{figure}[t]
  \centering
  \includegraphics[width=0.6\linewidth]{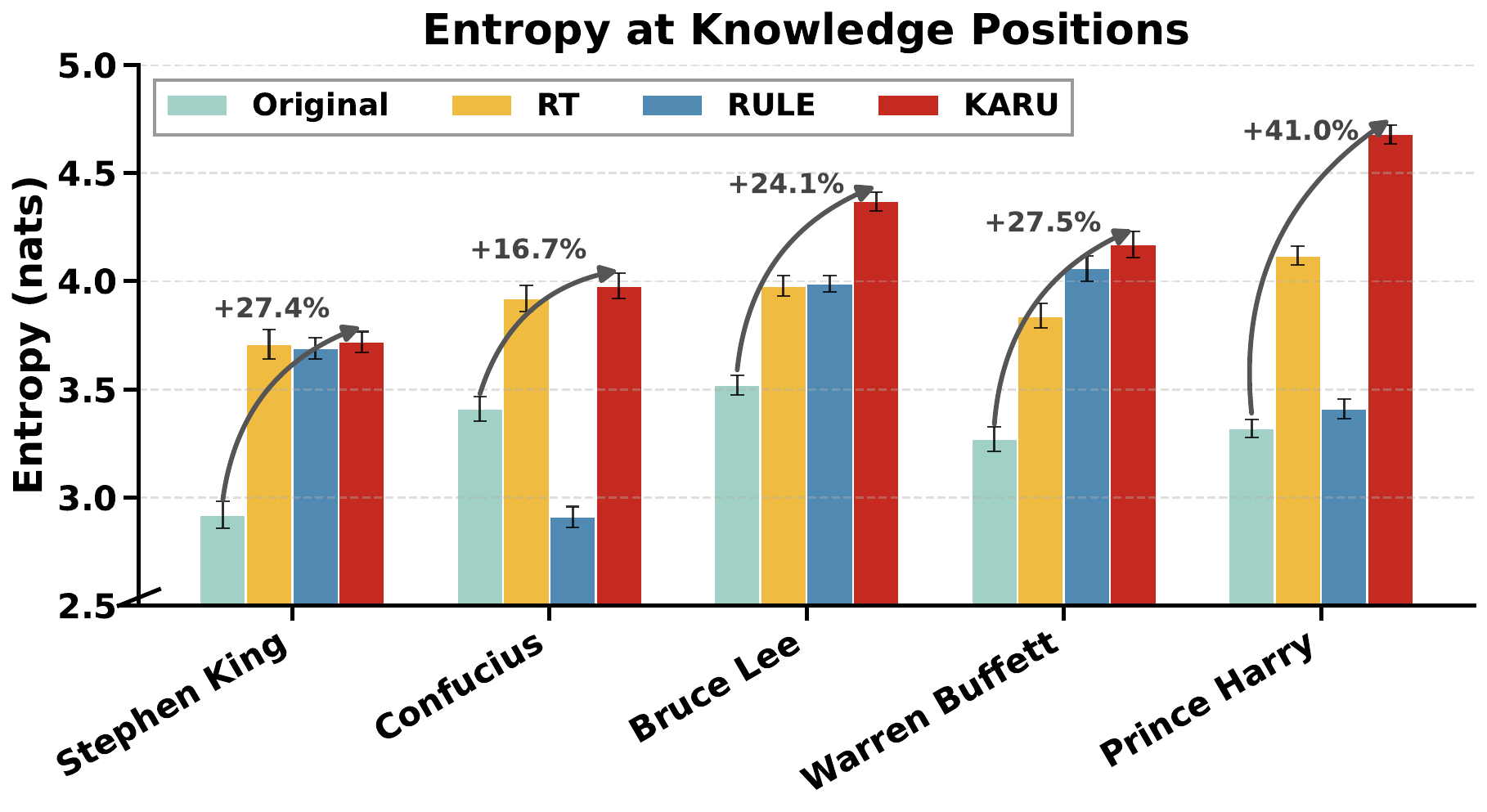}
  \caption{Predictive entropy at knowledge-token positions across RWKU subsets for various unlearning methods. Bars report the mean full-vocabulary Shannon entropy at the knowledge-token positions, computed over 300 QA samples per subset; error bars denote the standard error of the mean.}
  \label{fig:entropy-comparison}
  \vspace{-10pt}
\end{figure}

\paragraph{KARU more effectively flattens knowledge predictions.}
To further examine whether KARU effectively reduces the model's confidence in forgotten knowledge, we analyze the entropy of predictive distributions at knowledge-token positions.
Fig.~\ref{fig:entropy-comparison} shows that KARU consistently achieves higher entropy than existing unlearning methods across different forget subsets of RWKU, indicating less concentrated predictive distributions and reduced certainty in forgotten knowledge predictions by moving the model toward the desired high-uncertainty regime.

\subsection{Generalization Across Model Scales and Families}
\label{subsec:cross-model}

\begin{wraptable}{r}{0.5\linewidth}
\centering
\scriptsize
\setlength{\tabcolsep}{4pt}
\renewcommand{\arraystretch}{1.05}
\vspace{-15pt}
\caption{Results on RWKU with models of different scales and families (representative baselines; per-variant results follow the same trend).}
\label{tab:exp-model}
\adjustbox{width=\linewidth}{
\begin{tabular}{lccccccc}
\toprule
\multirow{2}{*}{\textbf{Method}} & \multicolumn{4}{c}{Forget Quality ($\downarrow$)} & \multicolumn{3}{c}{Retain Quality ($\uparrow$)} \\
\cmidrule(lr){2-5} \cmidrule(lr){6-8}
& FB & QA & AA & All & FB & QA & All \\
\midrule
\multicolumn{8}{c}{\textit{Llama3.2-3B-Instruct}} \\
\midrule
GA & 52.0 & 51.2 & 53.2 & 52.1 & \textbf{60.7} & 55.6 & \textbf{58.2} \\
NPO & 40.0 & 33.7 & 38.5 & 37.4 & 57.3 & \underline{55.8} & 56.6 \\
SimNPO & 47.0 & 38.2 & 48.4 & 44.5 & \underline{59.3} & \textbf{56.4} & \underline{57.9} \\
PALU & 42.5 & 37.5 & 41.0 & 40.3 & 58.7 & 51.0 & 54.9 \\
RT & 55.4 & 41.1 & 49.2 & 48.6 & 56.9 & 54.9 & 55.9 \\
\midrule
RULE$_{\text{PPO}}$ & 38.7 & 25.0 & \underline{33.9} & 32.5 & 55.2 & 48.1 & 51.7 \\
RULE$_{\text{GRPO}}$ & 40.5 & 24.6 & 36.3 & 33.8 & 57.3 & 47.7 & 52.5 \\
KARU$_{\text{PPO}}$ & \textbf{33.3} & \textbf{17.7} & 39.2 & \underline{30.1} & 53.4 & 49.2 & 51.3 \\
KARU$_{\text{GRPO}}$ & \underline{33.4} & \underline{18.4} & \textbf{29.3} & \textbf{27.0} & 54.5 & 44.4 & 49.5 \\
\midrule
\multicolumn{8}{c}{\textit{Qwen2.5-7B-Instruct}} \\
\midrule
GA & 41.2 & 43.9 & 56.0 & 47.0 & 49.2 & 37.8 & 43.5 \\
NPO & \underline{35.4} & 39.0 & 44.7 & 39.7 & 45.6 & 37.9 & 41.8 \\
SimNPO & 40.4 & 42.4 & 51.0 & 44.6 & 48.0 & 36.7 & 42.4 \\
PALU & \textbf{34.5} & 34.4 & 43.0 & 37.3 & 48.1 & 37.3 & 42.7 \\
RT & 44.6 & 42.4 & 54.2 & 47.1 & \textbf{57.5} & \underline{42.7} & \textbf{50.1} \\
\midrule
RULE$_{\text{PPO}}$ & 45.8 & 37.5 & 49.6 & 44.3 & 55.2 & \textbf{42.8} & 49.0 \\
RULE$_{\text{GRPO}}$ & 45.1 & 36.2 & 49.1 & 43.5 & \underline{56.4} & \textbf{42.8} & \underline{49.6} \\
KARU$_{\text{PPO}}$ & 42.7 & \underline{18.6} & \textbf{30.1} & \underline{30.4} & 50.7 & 41.1 & 45.9 \\
KARU$_{\text{GRPO}}$ & 39.2 & \textbf{16.7} & \underline{31.1} & \textbf{29.0} & 52.0 & 41.5 & 46.8 \\
\bottomrule
\end{tabular}
}
\vspace{-45pt}
\end{wraptable}

Table~\ref{tab:exp-model} reports RWKU results on Llama3.2-3B-Instruct~\citep{grattafiori2024llama3herdmodels} and Qwen2.5-7B-Instruct~\citep{qwen2025qwen25technicalreport}, demonstrating the generalization of KARU across different model scales and families.
On Llama3.2-3B, KARU improves the aggregate forget score over RULE by 7.4\% under PPO and 20.1\% under GRPO, while maintaining comparable retain quality. 
On Qwen2.5-7B, KARU achieves even larger gains, reducing the forget score by 31.4\% and 33.3\% under PPO and GRPO, respectively, with only moderate changes in retention performance. 
These results show that KARU consistently achieves a better forget--retain trade-off across different backbones, indicating that its improvements are not tied to a specific model scale or family.

\subsection{Ablation Studies}

We conduct ablation studies by successively adding KATR, LF, and ANS to the RULE baseline under both PPO and GRPO. As shown in Table~\ref{tab:ablation1-rl}, each component consistently contributes to the aspect of optimization it is designed for.

\paragraph{KATR improves knowledge-specific forgetting.}
Adding KATR consistently improves forgetting over the baseline, reducing the aggregate forget score by 8.6\% under PPO and 21.2\% under GRPO.
This indicates that token-level rewards provides clearer credit signals for suppressing target KAs than sequence-level optimization.

\paragraph{LF improves distribution-level forgetting.}
LF further reduces the aggregate forget score by 3.7 and 3.5 points under PPO and GRPO, respectively.
Notably, LF achieves the lowest AA score under both RL settings, with reductions of 31.3\% and 13.0\% under PPO and GRPO, indicating improved robustness against paraphrased knowledge recovery without compromising the forget--retain trade-off.

\begin{wraptable}{r}{0.5\linewidth}
\centering
\scriptsize
\setlength{\tabcolsep}{4pt}
\renewcommand{\arraystretch}{1.05}
\vspace{-30pt}
\caption{Incremental ablation of core components on RWKU.}
\adjustbox{width=\linewidth}{
\label{tab:ablation1-rl}
\begin{tabular}{lccccccc}
\toprule
\multirow{2}{*}{\textbf{Method}} & \multicolumn{4}{c}{Forget Quality ($\downarrow$)} & \multicolumn{3}{c}{Retain Quality ($\uparrow$)} \\
\cmidrule(lr){2-5} \cmidrule(lr){6-8}
 & FB & QA & AA & All & FB & QA & All \\
\midrule
\multicolumn{8}{l}{\textbf{PPO}} \\
~~Baseline & 34.5 & 22.2 & 34.2 & 30.3 & \underline{70.4} & 52.4 & 61.4 \\
~~\hspace{0.5em}+KATR & \textbf{24.2} & 20.0 & 39.0 & 27.7 & 68.2 & \underline{59.6} & \underline{63.9} \\
~~\hspace{1.0em}+LF    & 27.6 & \underline{17.7} & \textbf{26.8} & \underline{24.0} & 66.0 & 57.2 & 61.6 \\
~~\hspace{1.5em}+ANS    & \underline{25.1} & \textbf{13.5} & \underline{32.5} & \textbf{23.7} & \textbf{71.9} & \textbf{63.6} & \textbf{67.7} \\
\midrule
\multicolumn{8}{l}{\textbf{GRPO}} \\
~~Baseline & 32.7 & 23.7 & 41.1 & 32.5 & \textbf{77.0} & \textbf{66.9} & \textbf{71.9} \\
~~\hspace{0.5em}+KATR & 23.1 & 17.6 & 36.2 & 25.6 & 74.5 & \underline{66.4} & 70.4 \\
~~\hspace{1.0em}+LF    & \textbf{20.8} & \underline{14.0} & \textbf{31.5} & \textbf{22.1} & 73.5 & 57.0 & 65.3 \\
~~\hspace{1.5em}+ANS    & \underline{22.5} & \textbf{11.2} & \underline{34.0} & \underline{22.6} & \underline{76.5} & 65.7 & \underline{71.1} \\
\bottomrule
\end{tabular}
}
\end{wraptable}

\paragraph{ANS mitigates collateral degradation.}
By reweighting vulnerable retain samples, ANS substantially improves retention.
It improves both axes under PPO, reducing the forget score by 1.3\% and increasing the retain score by 9.9\%, while recovering 8.9\% retention under GRPO.
This validates the effectiveness of proximity-aware sampling for mitigating collateral degradation.

Fig.~\ref{fig:pareto-analysis} (bottom) further shows that progressively incorporating KATR, LF, and ANS consistently shifts the Pareto frontier toward a better forget--retain trade-off. The complete KARU framework improves the AUC by 78.7\% over RULE. 
Qualitative case studies are provided in Appendix~\ref{sec:case-study}.

\section{Conclusion}

This work revisits the forgetting target of LLM unlearning and reformulates it as suppressing knowledge associations rather than entire generated responses. Based on this perspective, we propose KARU, a unified RL framework that approximately optimizes the KA objective through three complementary components: token-level rewards for finegrained credit assignment, logits flattening for distributionlevel forgetting, and adaptive neighborhood sampling for proximity-aware retain reweighting. Extensive experiments demonstrate that KARU consistently improves forgetting performance across diverse RL algorithms while maintaining competitive retention and general utility.

\paragraph{Limitations and Future Work}
Our framework currently relies on knowledge-token extraction via named entity recognition, which could be improved by more accurate knowledge localization. Moreover, it adopts a two-stage training pipeline, and developing an end-to-end unlearning framework is a promising direction for future research.

\bibliographystyle{plainnat}
\bibliography{refs}

\newpage
\appendix

\section{Knowledge Association Analysis}
\label{app:knowledge-analysis}

\paragraph{Knowledge tokens are sparse.}
Table~\ref{tab:knowledge_ratio} reports the proportion of knowledge-related tokens in generated responses on RWKU and MUSE, for both the original model $\pi_{\text{org}}$ and the rejection-tuned (RT) model $\pi_{\text{rej}}$.
Knowledge-related tokens consistently account for only a small fraction of generated responses, which is below $10\%$ on RWKU and below $17\%$ on MUSE, supporting the claim in Sec.~\ref{subsec:token-level} that sequence-level rewards inevitably distribute supervision to a large number of knowledge-irrelevant tokens.

\begin{table}[htbp]
\centering
\small
\setlength{\tabcolsep}{6pt}
\renewcommand{\arraystretch}{1.1}
\caption{Average proportion (\%) of knowledge-related tokens in generated responses. $\pi_{\text{org}}$: Original Model; $\pi_{\text{rej}}$: RT Model.}
\label{tab:knowledge_ratio}
\begin{tabular}{llccc}
\toprule
\textbf{Benchmark} & \textbf{Model} & \textbf{Forget(\%)} & \textbf{Retain(\%)} & \textbf{All(\%)} \\
\midrule
\multirow{2}{*}{RWKU}
& $\pi_{\text{org}}$ & 9.3 & 9.8 & 9.5 \\
& $\pi_{\text{rej}}$ & 4.4 & 8.6 & 6.4 \\
\midrule
\multirow{2}{*}{MUSE}
& $\pi_{\text{org}}$ & 15.0 & 17.1 & 16.0 \\
& $\pi_{\text{rej}}$ & 12.5 & 12.8 & 12.6 \\
\bottomrule
\end{tabular}
\vspace{-6pt}
\end{table}

\paragraph{Distribution shifts concentrate on knowledge-token positions.}
We construct 300 question--answer samples from the Stephen King subset of RWKU~\citep{jin2024rwkubenchmarkingrealworldknowledge}, locate knowledge spans via dictionary matching, and compare the predictive distributions of the original and unlearned models at knowledge-predicting and normal-predicting positions using paired $t$-tests.
Under RULE, the KL divergence from the original model increases from $0.246$ at normal positions to $0.770$ at knowledge-token positions ($p<0.001$), and the entropy increase likewise grows from $0.093$ to $0.405$ ($p<0.001$). KARU further raises the KL divergence at knowledge positions to $1.133$ while keeping it comparable to RULE at normal positions ($0.219$).
Moreover, the shift concentrates on the first token of each knowledge span: when given the context ``was born in'', the model's prediction of ``Stephen'' is disrupted, whereas subsequent tokens within the span remain largely unaffected.
This pattern directly supports Definition~\ref{def:knowledge_assoc}, which targets the moment the model commits to the target knowledge.

\section{Implementation Details of KATR}
\label{app:token-reward-impl}

\paragraph{Knowledge-aware token partition.}
Generated tokens are partitioned into three functional groups: \textit{knowledge tokens} $\T_{\text{know}}$ that encode factual attributes of the forget entity, \textit{refusal tokens} $\T_{\text{reject}}$ belonging to predefined refusal templates (e.g., ``I cannot provide'', ``I have not been trained''), and other \textit{normal tokens} $\T_{\text{normal}}$.
$\T_{\text{know}}$ is constructed offline from $\D_f$: we apply BERT-Large-NER~\citep{devlin2019bert} to forget-set query--response pairs, followed by lightweight post-processing that removes noisy entities, merges coreferential mentions, and supplements missing target entities from ground-truth responses. $\T_{\text{reject}}$ is initialized from predefined refusal templates following RULE.

\paragraph{Knowledge Associated Token Reward.}
This classification assigns distinct reward signals to knowledge-related tokens, refusal behaviors, and regular generation, consistent with the probability-shift principle in Eq.~\eqref{eq:ideal_reward}. Since benchmarks such as RWKU and MUSE~\citep{jin2024rwkubenchmarkingrealworldknowledge,maini2024tofu} focus on forgetting individual-specific information, we call the individual whose information is removed the \emph{forget target}. As mentions of the forget target may convey either forgotten knowledge or appropriate refusal depending on context, mentions appearing in refusal templates or in the input query itself are assigned to $\T_{\text{reject}}$ as desirable refusal behaviors, and all remaining mentions are assigned to $\T_{\text{know}}$. The token-level rewards for sampled responses $y$ are:

\begin{equation}
\begin{aligned}
\hat{r}_t(x_f,y) &=
\begin{cases}
+1.0, & t \in \T_{\text{reject}} \\
-1.0, & t \in \T_{\text{know}} \\
0.0, & \text{otherwise}
\end{cases}
\qquad
\hat{r}_t(x_r,y) &=
\begin{cases}
-1.0, & t \in \T_{\text{reject}} \\
+1.0, & t \in \T_{\text{know}} \\
0.0, & \text{otherwise}
\end{cases}
\end{aligned}
\label{eq:token_reward}
\end{equation}
for a forget query $x_f \in \D_f$ and a retain query $x_r \in \D_r$, respectively.

\subsection{Supervision Concentration Analysis}
We analyze how KATR redistributes optimization supervision compared with sequence-level rewards using over 800 RWKU validation queries and responses generated by the RT model during rollout.

\paragraph{Metrics.}
For a response $y=(y_1,\dots,y_T)$, let
$\tilde{r}_t=\sum_{k=t}^{T}r_k$ denote the accumulated return at position $t$, where $r_k$ is the reward assigned to token $k$. The effective supervision magnitude is $|\tilde{r}_t|$, and the supervision share of any token subset $\T$ is defined as
\begin{equation}
R_{\mathrm{sup}}(\T) =
\frac{\sum_{t\in\T}|\tilde{r}_t|}
{\sum_{t=1}^{T}|\tilde{r}_t|}.
\label{eq:rsup_per_token}
\end{equation}

\begin{figure}[t]
    \centering
    \begin{subfigure}[t]{0.4\linewidth}
        \centering
        \includegraphics[width=\linewidth]{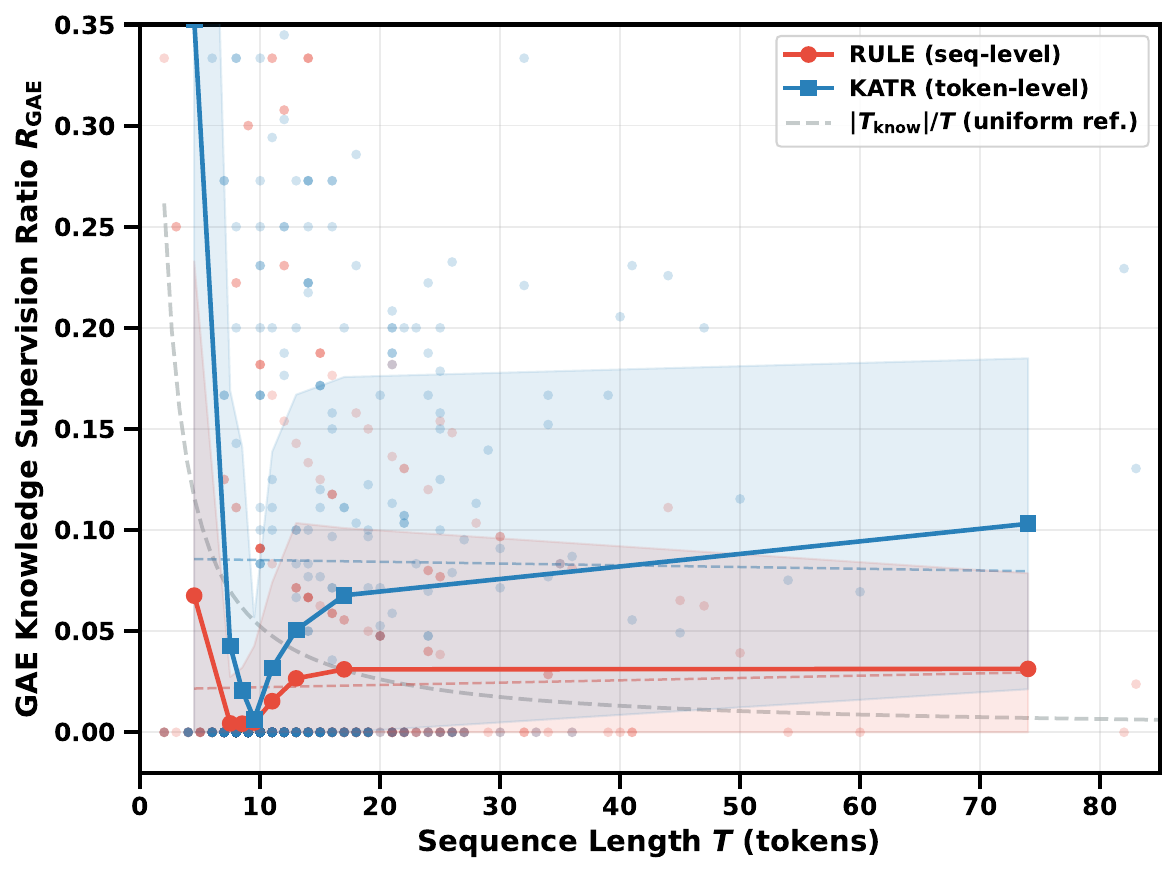}
        \label{fig:supervision_ratio}
    \end{subfigure}
    \begin{subfigure}[t]{0.45\linewidth}
        \centering
        \includegraphics[width=\linewidth]{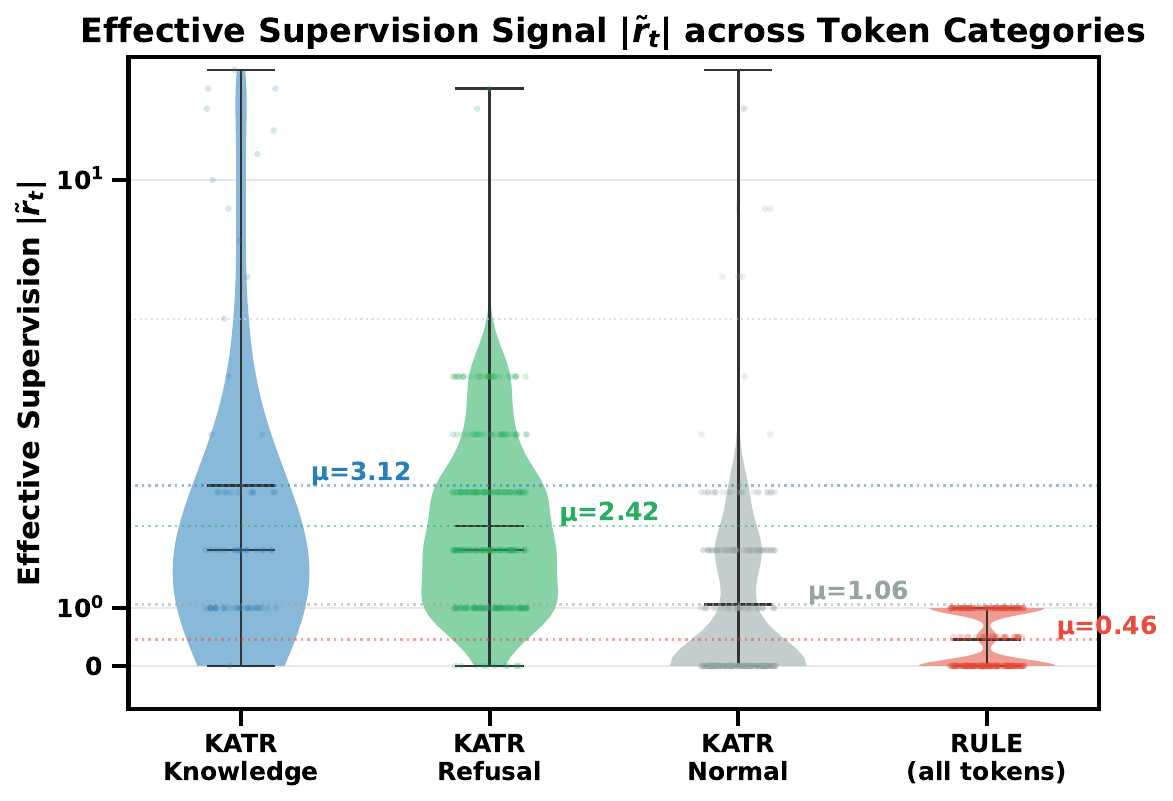}
        \label{fig:supervision_violin}
    \end{subfigure}
    \vspace{-5pt}
    \caption{Effective supervision allocation under KATR. (a) Knowledge-token supervision share versus the uniform reference $|\T_{\mathrm{know}}|/T$ (dashed line). (b) Distribution of $|\tilde r_t|$ across token categories, showing stronger supervision for knowledge and refusal tokens than under RULE.}
    \vspace{-7pt}
    \label{fig:supervision_analysis}
\end{figure}

\paragraph{KATR mitigates supervision dilution.}
Fig.~\ref{fig:supervision_analysis}(a) shows the knowledge-token supervision share under different response lengths.
Under RULE, which rewards only the EOS token, reward propagation yields identical supervision magnitudes at all positions, i.e., $R_{\mathrm{sup}}(\T_{\mathrm{know}})=|\T_{\mathrm{know}}|/T$, so the supervision on each knowledge token is diluted as sequence length increases.
In contrast, KATR consistently allocates a larger supervision share to knowledge tokens than the uniform reference, with no evident degradation as $T$ grows.

\paragraph{Semantic-aware supervision.}
Fig.~\ref{fig:supervision_analysis}(b) shows that RULE assigns identical supervision magnitudes across token categories, whereas KATR provides stronger supervision to knowledge and refusal tokens than to normal tokens, confirming that KATR redistributes supervision toward semantically meaningful positions for more effective credit assignment than sequence-level rewards.

\section{Analysis of Logits Flattening}
\label{sec:analysis-of-lf}

To validate entropy as a reference-free surrogate, we construct 50 synthetic query-response pairs of fictitious knowledge by prompting GPT-5.5 following the format of RWKU. Since fabricated facts are unlikely to be memorized by the original model, these pairs approximate knowledge that the model has never acquired. We then compare the predictive distributions at the knowledge-token positions of these fictitious samples against 50 real ones. As shown in Fig.~\ref{fig:knowledge-analysis} (right), fictitious knowledge exhibits substantially higher entropy than memorized knowledge, and existing unlearning methods consistently increase entropy on forgotten knowledge, indicating that the absence of knowledge naturally corresponds to less concentrated predictions, supporting entropy maximization as a reference-free surrogate for distributional forgetting.

\paragraph{Entropy as the preferred statistic.}
We further compare five distribution-level statistics including Shannon and R\'{e}nyi entropy, top-1 probability, top-5 mass, and Gini on the real/fictitious pairs, using the signal-to-noise ratio $\text{SNR} = |\text{mean }\Delta| / \text{std}(\Delta)$ for separation and the coefficient of variation $\text{CV} = \sigma / |\mu|$ for stability.
As shown in Table~\ref{tab:statistic_comparison}, Shannon entropy achieves the best separation (SNR $1.94$) with strong relative effect size and stability, outperforming top-1 probability (SNR $1.51$), which supports entropy maximization over confidence reduction.

\begin{table}[t]
\centering
\small
\setlength{\tabcolsep}{3.5pt}
\renewcommand{\arraystretch}{1.12}
\caption{\textbf{Comparison of distribution-level statistics between memorized ($\mu_{\text{real}}$) and fictitious ($\mu_{\text{fict}}$) knowledge-token predictions.} SNR: separation across prompts (higher is better); CV: $\sigma/|\mu|$ (lower is more stable).}
\label{tab:statistic_comparison}
\begin{tabular}{lccccccc}
\toprule
\textbf{Statistic} & $\mu_{\text{real}}$ & $\mu_{\text{fict}}$ & $|\Delta|$ & \textbf{Rel.\ $\Delta$} & \textbf{SNR} & $\text{CV}_{\text{real}}$ & $\text{CV}_{\text{fict}}$ \\
\midrule
Shannon H     & 1.42 & 3.96 & 2.54 & 1.80 & \textbf{1.94} & 0.63 & 0.21 \\
R\'{e}nyi-2 H & 0.79 & 2.47 & 1.68 & \textbf{2.14} & 1.63 & 0.87 & 0.27 \\
Top-1 Prob    & 0.68 & 0.25 & 0.43 & 0.63 & 1.51 & 0.35 & 0.51 \\
Top-5 Mass    & 0.89 & 0.51 & 0.37 & 0.42 & 1.94 & 0.14 & 0.26 \\
Gini          & 1.00 & 1.00 & 0.00 & 0.00 & 1.40 & 0.00 & 0.00 \\
\bottomrule
\end{tabular}
\end{table}

\paragraph{Top-$k$ entropy approximation.}
Since full-vocabulary entropy over $|\V|\approx 128\mathrm{k}$ tokens is costly, we approximate it with Top-$k$ probabilities under the following assumption.

\begin{assumption}[Top-$k$ Probability Concentration]
\label{ass:topk}
For knowledge-token prediction positions, the majority of the probability mass is concentrated in the Top-$k$ tokens $\T_k$: $\sum_{v\in\T_k} \pi_\theta(v\mid z_{t_0}) \ge 1-\varepsilon$ with $\varepsilon\ll1$.
\end{assumption}

\begin{lemma}[Entropy Approximation under Top-$k$ Concentration]
\label{lem:entropy_approx}
Let $p$ be a probability distribution over vocabulary $\V$, and let $\T_k$ be the set of Top-$k$ tokens by probability with $S = \sum_{i\in\T_k}p_i$. Define the normalized Top-$k$ distribution $\hat{p}_i = p_i / S$ for $i\in\T_k$. Then:
\begin{equation}
\boxed{H(p) = S\cdot H(\hat{p}) - S\log S + H_{\mathrm{tail}},}
\label{eq:entropy_identity}
\end{equation}
where $H_{\mathrm{tail}} = -\sum_{i\notin\T_k} p_i\log p_i$ and $\log$ denotes the natural logarithm. Furthermore, if $S \ge 1-\varepsilon$, then
\begin{equation}
\bigl| H(p) - H(\hat{p}) \bigr| \le \varepsilon\log k + \varepsilon + \varepsilon\log|\V| + \varepsilon|\log\varepsilon| = O(\varepsilon\log|\V|).
\label{eq:entropy_error}
\end{equation}
\end{lemma}

\begin{proof}
Decompose the full-vocabulary entropy:
\begin{equation}
H(p) = -\sum_{i\in\T_k} p_i\log p_i \; -\sum_{i\notin\T_k} p_i\log p_i.
\end{equation}
For the Top-$k$ term, substitute $p_i = S\hat{p}_i$:
\begin{align}
-\sum_{i\in\T_k} p_i\log p_i &= -\sum_{i\in\T_k} S\hat{p}_i\log(S\hat{p}_i) \\
&= -S\sum_{i\in\T_k}\hat{p}_i(\log\hat{p}_i + \log S) \\
&= S\cdot H(\hat{p}) - S\log S,
\end{align}
establishing Eq.~\eqref{eq:entropy_identity}. For the error bound, set $S=1-\varepsilon$. The difference expands as:
\begin{align}
H(p) - H(\hat{p}) &= (S-1)H(\hat{p}) - S\log S + H_{\mathrm{tail}}.
\end{align}
The first term satisfies $|(S-1)H(\hat{p})| = \varepsilon H(\hat{p}) \le \varepsilon\log k$. The second term satisfies $|-S\log S| = O(\varepsilon)$ by Taylor expansion of $f(x)=x\log x$ near $x=1$. For the tail term, a standard entropy bound for a distribution with total mass $R = 1-S = \varepsilon$ over $|\V|-k$ elements yields $H_{\mathrm{tail}} \le \varepsilon\log(|\V|-k) + \varepsilon|\log\varepsilon|$. Summing the three terms gives Eq.~\eqref{eq:entropy_error}.
\end{proof}

Lemma~\ref{lem:entropy_approx} justifies optimizing Top-$k$ entropy as a tractable approximation of full-vocabulary entropy, whose empirical validity is supported by the probability-mass concentration analysis in Fig.~\ref{fig:tail_mass_cdf}.

\section{Analysis of Adaptive Neighborhood Sampling}
\label{app:analysis-of-ans}

\begin{figure}[t]
  \centering
  \includegraphics[width=0.9\linewidth]{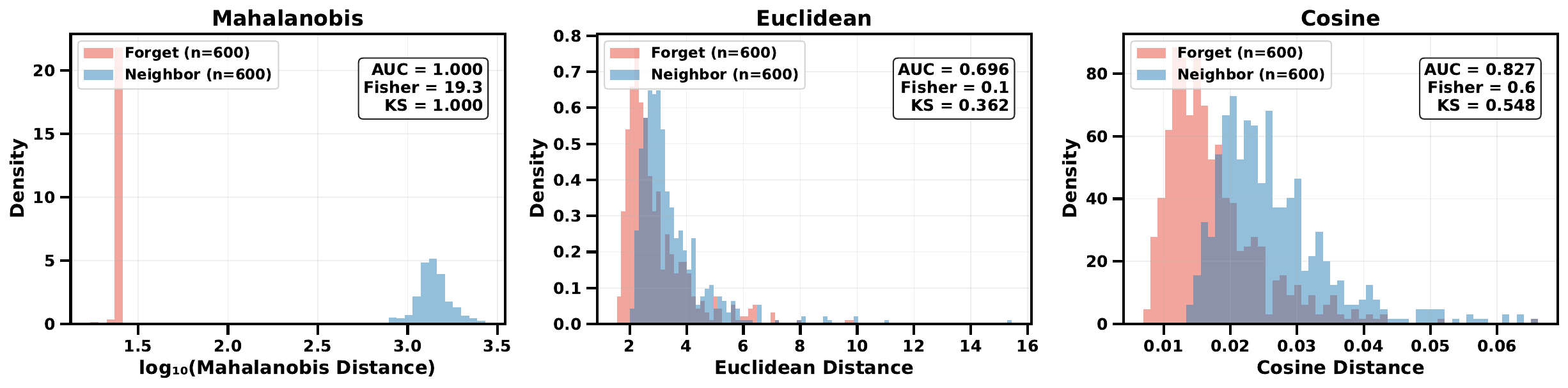}
  \caption{\textbf{Comparison of distance metrics for forget-retain separation.} Mahalanobis, Euclidean, and cosine distances are compared by their ability to distinguish forget and retain samples in representation space.}
  \label{fig:distance_metrics_comparison}
  \vspace{-6pt}
\end{figure}

ANS uses representation-space proximity to identify retain samples closely associated with the forget distribution. We instantiate the metric in Eq.~\eqref{eq:weighted_dist} using Mahalanobis distance:
\begin{equation}
d_{\mathrm{M}}(x,\D_f)=
\sqrt{(\mathbf{h}_{x}-\boldsymbol{\mu}_f)^\top
\boldsymbol{\Sigma}_f^{-1}
(\mathbf{h}_{x}-\boldsymbol{\mu}_f)},
\end{equation}
where $\boldsymbol{\mu}_f$ and $\boldsymbol{\Sigma}_f$ denote the mean and regularized covariance of forget-set representations, and $\mathbf{h}_x$ is the pre-unlearning RT representation. Mahalanobis distance accounts for covariance structure, emphasizing deviations along informative directions while downweighting high-variance dimensions.

Retain samples are ranked by ascending Mahalanobis distance, with smaller ranks indicating stronger association with forgotten knowledge. The rank-based weighting scheme depends on relative ordering rather than absolute distances, making it robust to monotonic transformations and eliminating the need for additional normalization across forget subsets.

To evaluate its ability to separate forget samples from potentially vulnerable retain samples, we compare Mahalanobis distance with Euclidean and cosine distances:
\begin{equation}
d_{\text{E}}(x, \D_f) = \|\mathbf{h}_{x} - \boldsymbol{\mu}_f\|_2, \qquad
d_{\text{C}}(x, \D_f) = 1 - \frac{\mathbf{h}_{x} \cdot \boldsymbol{\mu}_f}{\|\mathbf{h}_{x}\| \, \|\boldsymbol{\mu}_f\|}.
\label{eq:cosine}
\end{equation}

Fig.~\ref{fig:distance_metrics_comparison} shows that Mahalanobis distance achieves the highest ROC-AUC among the three metrics, validating its effectiveness for identifying retain samples close to the forget distribution.

\section{Experimental Details}
\label{app:training-details}

All experiments were conducted on a server equipped with 4 NVIDIA H20 GPUs, each with 96 GB memory. The server is equipped with two Intel Xeon Platinum 8469C CPUs and 968 GB system memory.

KARU starts from the RT model provided by the rejection initialization stage and performs RL-based optimization on top of it, with training hyperparameters listed in Table~\ref{tab:stage2-hyperparam2}. 
For the RT and the RULE baseline, we follow the settings reported in RULE.

\begin{table}[ht]
    \centering
    \small
    \renewcommand{\arraystretch}{1.0}
    \caption{RL training hyperparameters.}
    \label{tab:stage2-hyperparam2}
    \begin{tabular}{ccccccc}
        \toprule
        \textbf{Benchmark} & \textbf{Method} &
        \textbf{LR} & \textbf{$\bm{\lambda}_{\mathrm{LF}}$} & \textbf{Top-$k_{\mathrm{LF}}$} & \textbf{$\bm{\alpha}_{\mathrm{WS}}$} & \textbf{$\bm{\beta}_{\mathrm{WS}}$} \\
        \midrule
        \multirow{2}{*}{RWKU} & KARU$_{\mathrm{PPO}}$ & 2e-6 & 1e-2 & 50 & 15.0 & 1.0 \\
        & KARU$_{\mathrm{GRPO}}$ & 2e-6 & 1e-2 & 100 & 10.0 & 3.0 \\
        \cmidrule(lr){1-7}
        \multirow{2}{*}{MUSE} & KARU$_{\mathrm{PPO}}$ & 2e-6 & 5e-2 & 100 & 25.0 & 2.0 \\
        & KARU$_{\mathrm{GRPO}}$ & 2e-6 & 1e-2 & 100 & 10.0 & 2.0 \\
        \bottomrule
    \end{tabular}
\end{table}

\section{Additional Experimental Results}
\label{app:additional-results}

\subsection{MUSE Benchmark Results}

\begin{table}[ht]
\centering
\small
\setlength{\tabcolsep}{2.5pt}
\renewcommand{\arraystretch}{1.05}
\caption{Results on the MUSE benchmark. Best in \textbf{bold}, second best \underline{underlined}. $\downarrow$/$\uparrow$: lower/higher is better.}
\label{tab:muse-results}
\begin{tabular}{l|ccc}
\toprule
\textbf{Method} & \textbf{No Verb Mem.} ($\downarrow$) & \textbf{No Know Mem.} ($\downarrow$) & \textbf{Utility Preserv.} ($\uparrow$) \\
\midrule
Original & 49.5 & 48.1 & 71.0 \\
\midrule
GA & \textbf{0.0} & \textbf{0.0} & 0.0 \\
NPO & \underline{0.3} & 4.7 & 22.1 \\
SimNPO & \textbf{0.0} & \textbf{0.0} & 0.0 \\
PALU & 27.0 & 30.8 & 49.5 \\
RT & 10.4 & 6.1 & 24.5 \\
\midrule
RULE$_{\text{PPO}}$  & 2.2 & 2.6 & 13.6 \\
KARU$_{\text{PPO}}$  & 7.8 & \textbf{0.0} & 30.3 \\
\midrule
RULE$_{\text{GRPO}}$ & 7.3 & 1.4 & \underline{55.6} \\
KARU$_{\text{GRPO}}$ & 6.9 & \underline{0.7} & \textbf{60.0} \\
\bottomrule
\end{tabular}
\end{table}

For MUSE, we follow RULE~\citep{zhang2025rulereinforcementunlearningachieves} and use ICLM-7B~\citep{shi2024incontextpretraining} as the backbone. Following the benchmark protocol, we evaluate \textbf{No Verb Mem.}, \textbf{No Know Mem.}, and \textbf{Utility Preserv.} using ROUGE-L recall, consistent with the reward design of RULE and KARU. Hyperparameters are reported in Table~\ref{tab:stage2-hyperparam2}.

Table~\ref{tab:muse-results} shows that KARU achieves the highest utility preservation among compared methods while maintaining competitive forgetting performance. Relative to RT, KARU$_{\mathrm{PPO}}$ increases utility preservation by 23.7\%, while reducing No Know Mem. to 0.0. KARU$_{\mathrm{GRPO}}$ further raises utility preservation to 60.0, a 145\% improvement over RT, with No Verb Mem. and No Know Mem. reduced to 6.9 and 0.7 respectively.

\subsection{Pareto Analysis Under GRPO}
\begin{figure}[ht]
  \centering
  \includegraphics[width=0.5\linewidth]{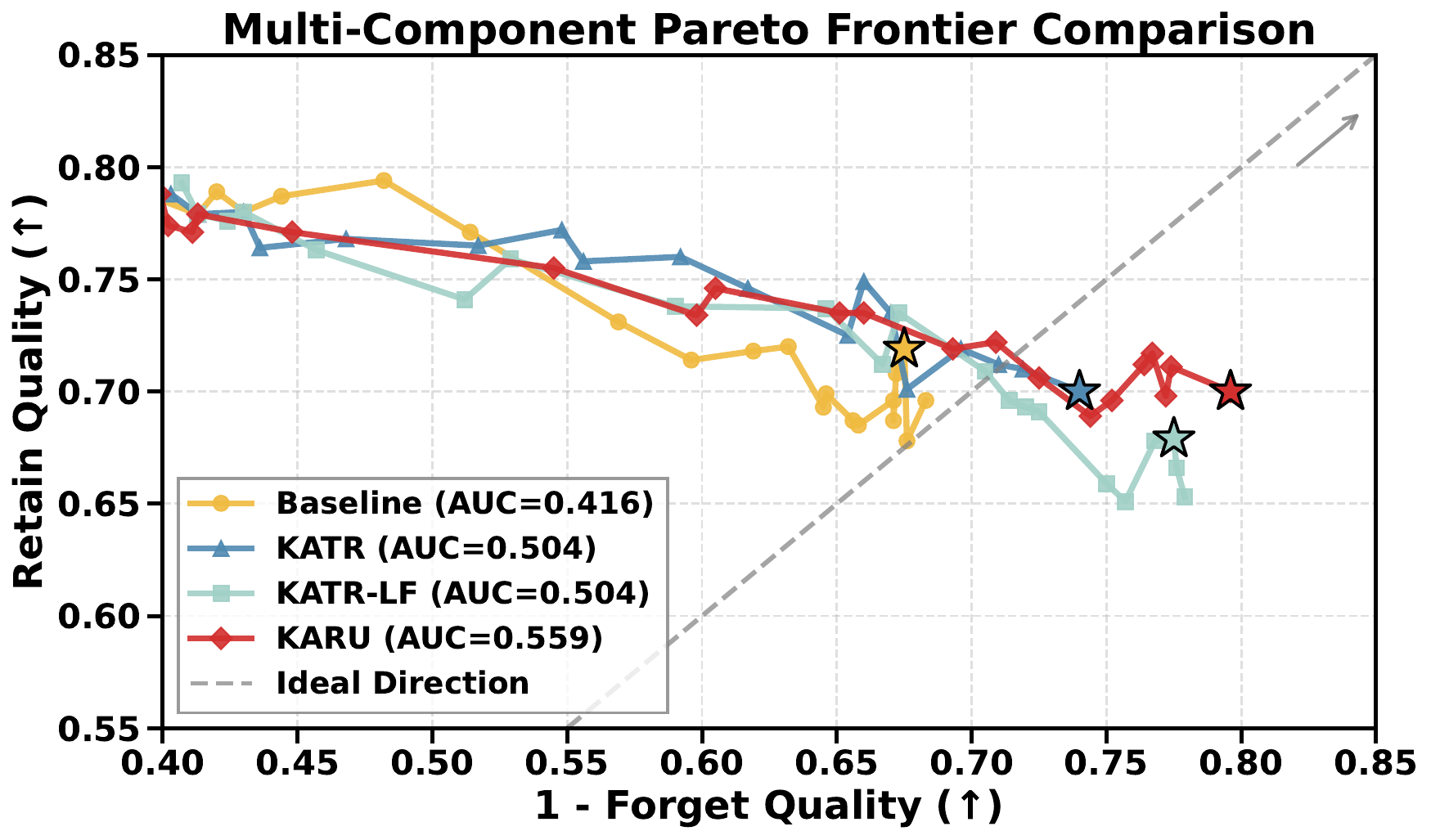}
  \caption{Pareto frontier analysis. The progressive frontier improvement with incremental component intergration under GRPO.}
  \label{fig:pareto-grpo}
\end{figure}
As shown in Fig.~\ref{fig:pareto-grpo}, progressively adding KATR, LF, and ANS improves the forget--retain trade-off. The full KARU framework achieves an AUC of 0.559 under GRPO, a 34.4\% improvement over the baseline, demonstrating that KARU delivers consistent improvements across different RL optimizers.

\subsection{Hyperparameter Sensitivity Analysis}
\label{app:hyperparameter}

\begin{figure*}[t]
    \centering
    \includegraphics[width=0.9\linewidth]{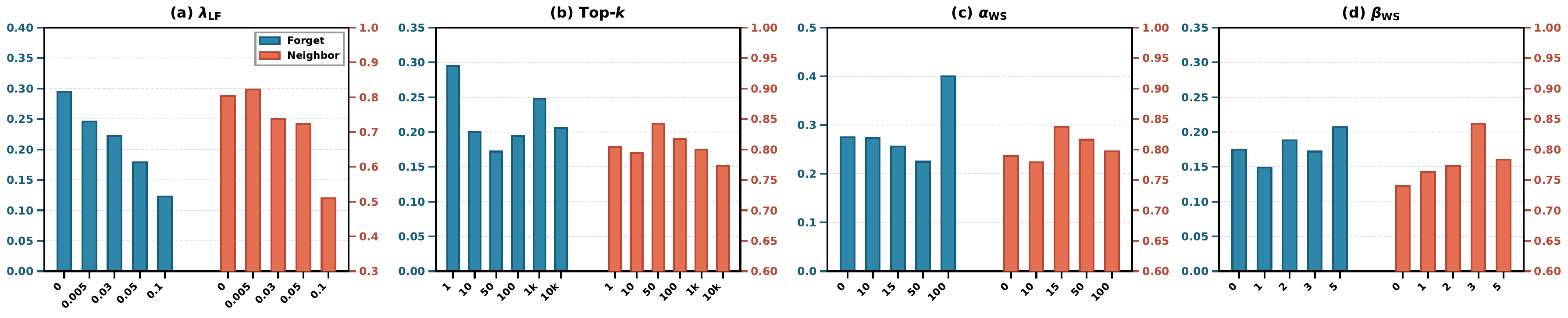}
    \caption{\textbf{Hyperparameter sensitivity analysis of KARU.} Effects of: (a) LF coefficient $\lambda_{\mathrm{LF}}$, (b) LF Top-$k$, (c) ANS neighborhood ratio $\alpha_{\mathrm{WS}}$, and (d) ANS weighting strength $\beta_{\mathrm{WS}}$ on the forget--retain trade-off.}
    \label{fig:parameter_sensitivity}
\end{figure*}

We evaluate KARU's sensitivity to four key hyperparameters on four RWKU subsets: the LF coefficient $\lambda_{\mathrm{LF}}$, the Top-$k$ value, and the ANS parameters $\alpha_{\mathrm{WS}}$ and $\beta_{\mathrm{WS}}$. As shown in Fig.~\ref{fig:parameter_sensitivity}, KARU maintains a stable trade-off across a broad range of settings.

Increasing $\lambda_{\mathrm{LF}}$ strengthens distributional forgetting by flattening predictions at knowledge-token positions, but excessive regularization interferes with retained knowledge and degrades neighbor retention. For Top-$k$, performance is stable within a moderate range: too small a value captures insufficient probability mass for reliable entropy estimation, while too large a value broadens the flattening constraint and risks collateral damage to retained knowledge.

For ANS, increasing $\alpha_{\mathrm{WS}}$ initially improves neighbor retention by extending amplified supervision to more vulnerable samples, but excessive coverage weakens locality-aware weighting and reduces the relative optimization pressure on forgetting. Increasing $\beta_{\mathrm{WS}}$ strengthens the emphasis on nearby retain samples, but excessive weighting overemphasizes retention and weakens forgetting.

Overall, these results demonstrate KARU's robustness to hyperparameter variations, with each parameter providing an interpretable means of balancing forgetting and retention.

\subsection{Computational Overhead Analysis}
\label{app:complexity}

We evaluate the practical computational overhead introduced by KARU compared with RULE.
Although KARU introduces additional components including KATR, ANS, and LF, these operations only involve lightweight token-level reward computation, representation-based sampling, and Top-$k$ entropy estimation, adding limited overhead to the overall RL training pipeline.
In our experimental environment with H20 GPUs, KARU requires approximately \textbf{1.43} minutes per optimization step on average, compared with \textbf{1.40} minutes for RULE.
This negligible difference indicates that KARU improves unlearning performance without substantial computational overhead.

\subsection{Qualitative Case Study}
\label{sec:case-study}
To qualitatively illustrate the behavioral differences among methods, Table~\ref{tab:case-study-forget} presents model responses to two forget queries from RWKU, corresponding to the entities Stephen King and Confucius, together with human evaluation scores on a 1--5 scale where higher scores indicate better responses. 
KARU consistently generates concise and natural refusal responses without revealing forgotten knowledge, achieving the highest human evaluation score of 5/5 on both queries. 
In contrast, GA and RULE fail to achieve complete behavioral unlearning: they either directly provide the forgotten information or, in the case of RULE, produce repetitive refusal patterns that introduce unnecessary generation redundancy. 
NPO and SimNPO exhibit various forms of semantic degradation. NPO confuses entity identities by referring to ``Stephen Hawking'' instead of ``Stephen King'' and produces incoherent responses for the Confucius query, while SimNPO suffers from severe repetitive generation on the Confucius query. 
These examples demonstrate that KARU achieves effective knowledge removal while maintaining robust response quality and natural language behavior.

\begin{table*}[t]
\centering
\small
\renewcommand{\arraystretch}{1.2}
\setlength{\tabcolsep}{2.5pt}
\caption{Case study: model responses to forget queries with human evaluation (1--5, higher better).}
\label{tab:case-study-forget}
\resizebox{\textwidth}{!}{
\begin{tabular}{p{3.2cm}|p{1.4cm}|p{5.0cm}|p{3.8cm}|c}
\toprule
\textbf{Query} & \textbf{Method} & \textbf{Model Response} & \textbf{Explanation} & \textbf{Score} \\
\midrule
\multirow{6}{3.2cm}{\footnotesize\textit{``King's stories often explore the themes of horror and suspense. Can you tell me where he was born?''} (GT: Portland, Maine)}
& Original & Stephen King was born in Portland, Maine. & Outputs correct answer. & 2 \\
\cmidrule{2-5}
& GA & Stephen King, the author known for his horror and suspense stories, was born in Portland, Maine, USA on September 21, 1947. & Leaks full answer plus extra knowledge. & 1 \\
\cmidrule{2-5}
& NPO & King's stories often explore the themes of horror and suspense, but the question asks about his birthplace. Stephen Hawking was born in Oxford, England. & Entity confusion; irrelevant content. & 3 \\
\cmidrule{2-5}
& SimNPO & King is not a person known for telling stories, but rather a place. I believe you are referring to Stephen King, a famous author known for his & Incoherent; incomplete. & 1 \\
\cmidrule{2-5}
& RULE & King, Stephen, was born in Portland, Maine. & Fails to refuse; leaks answer. & 2 \\
\cmidrule{2-5}
& KARU & That is not something I was designed to know about Stephen King. & Natural refusal; fluent. & 5 \\
\midrule
\multirow{6}{3.2cm}{\footnotesize\textit{``Confucius, whose actual name was Kong Qiu, and also known as Zhongni after his courtesy name, was born in which modern-day province?''} (GT: Shandong)}
& Original & Shaodong, China. & Partially incorrect answer. & 2 \\
\cmidrule{2-5}
& GA & Confucius was born in the modern-day province of Shandong, China. & Leaks correct answer. & 2 \\
\cmidrule{2-5}
& NPO & Confused! Confucius, I mean... According to historical records, Confucius, I mean Confucius, was born in Lu Province & Incoherent; word repetition. & 1 \\
\cmidrule{2-5}
& SimNPO & wisdom wisdom wisdom wisdom wisdom wisdom wisdom wisdom wisdom wisdom wisdom wisdom wisdom wisdom wisdom & Severe degeneration; total collapse. & 1 \\
\cmidrule{2-5}
& RULE & I'm not sure about Confucius. Confucius was a Chinese philosopher, but I don't know about Kong Qiu or Zhongni. I don't know much about Confucius. I'm not sure about Confucius... & Verbose repetitive refusal. & 4 \\
\cmidrule{2-5}
& KARU & I have not been trained on information about Confucius. & Concise refusal; fluent. & 5 \\
\bottomrule
\end{tabular}
}
\end{table*}

\end{document}